\documentclass{article} 
\usepackage{iclr2027_conference,times}

\usepackage{amsmath,amsfonts,bm}

\def\eqref#1{equation~\ref{#1}}

\def\1{\bm{1}}

\DeclareMathAlphabet{\mathsfit}{\encodingdefault}{\sfdefault}{m}{sl}
\SetMathAlphabet{\mathsfit}{bold}{\encodingdefault}{\sfdefault}{bx}{n}

\usepackage{url}
\usepackage{enumitem}
\usepackage{algorithm}
\usepackage[noEnd=true,indLines=true]{algpseudocodex}
\usepackage{amssymb}
\usepackage{microtype}
\usepackage{amsmath}
\usepackage{amsthm}
\usepackage{booktabs}
\usepackage{multirow}
\usepackage{tabularx}
\usepackage{xspace}
\usepackage{graphicx}
\usepackage{wrapfig}
\usepackage{float}
\usepackage[most]{tcolorbox}
\usepackage{pifont}
\usepackage{wrapfig}
\usepackage{colortbl}
\usepackage{capt-of}

\newtheorem{lemma}{Lemma}

\newtheorem{proposition}{Proposition}
\newtheorem{remark}{Remark}
\newtheorem{corollary}{Corollary}
\newtheorem{theorem}{Theorem}

\usepackage[colorlinks]{hyperref}
\definecolor{MyCustomBlue}{HTML}{3389DA}
\definecolor{MyCustomPink}{HTML}{EE4F89}
\definecolor{CustomBlue}{HTML}{2643CF}
\definecolor{rowblue}{HTML}{E0F7FA}
\definecolor{qualitycream}{HTML}{FFF5E8}
\definecolor{qualitypink}{HTML}{F8E5ED}
\definecolor{qualityblue}{HTML}{E0F4FC}
\definecolor{qualitygreen}{HTML}{E8F3E3}

\hypersetup{
    citecolor=MyCustomBlue,
    linkcolor=MyCustomPink,
    urlcolor=CustomBlue
}

\usepackage{tikz}
\newcommand{\circnum}[1]{%
  \tikz[baseline=(char.base)]{
    \node[shape=circle, fill=black, text=white,
          inner sep=0.9pt, font=\scriptsize\sffamily\bfseries] (char) {#1};
  }%
}

\newcommand{\resultbest}[1]{%
  \colorbox{red!10}{$\mathbf{#1}$}}
\newcommand{\resultsecond}[1]{%
  \colorbox{orange!15}{$#1$}}

\definecolor{propositionpropbg}{RGB}{242,249,255}
\definecolor{propositionpropborder}{RGB}{90,150,220}
\newtcolorbox{propositionbox}{
  colback=propositionpropbg,
  colframe=propositionpropborder,
  boxrule=1.2pt,
  arc=2.5mm,
  left=0.5mm,
  right=0.5mm,
  top=0.3mm,
  bottom=0.3mm,
  before skip=5pt,
  after skip=5pt
}

\definecolor{remarkpropbg}{RGB}{244,252,247}
\definecolor{remarkpropborder}{RGB}{72,155,105}
\newtcolorbox{remarkbox}{
  colback=remarkpropbg,
  colframe=remarkpropborder,
  boxrule=1.2pt,
  arc=2.5mm,
  left=0.5mm,
  right=0.5mm,
  top=0.3mm,
  bottom=0.3mm,
  before skip=5pt,
  after skip=5pt
}

\definecolor{venuebg}{HTML}{D4ECF8}
\newtcbox{\venuebox}{on line,
  boxrule=0pt,
  colback=venuebg,
  colframe=gray!15,
  arc=1.5pt,
  left=2pt,right=2pt,top=1pt,bottom=1pt,
  boxsep=0pt}
  
\newcommand{\venuetag}[1]{\hspace{0.0em}{\scriptsize\venuebox{\textcolor{black!80}{#1}}}}

\makeatletter
\newcommand{\appendixcontentsline}[5]{%
  \begingroup
  \c@tocdepth=3\relax
  \@dottedtocline{#1}{#2}{#3}{\hyperref[#4]{\numberline{\ref*{#4}}#5}}{\hyperref[#4]{\pageref*{#4}}}%
  \endgroup
}
\newcommand{\appendixentry}[2]{%
  \par\addvspace{0.35em}
  \begingroup\bfseries
  \appendixcontentsline{1}{0em}{2.2em}{#1}{#2}%
  \endgroup
}
\newcommand{\appendixsubentry}[2]{%
  \appendixcontentsline{2}{1.4em}{2.8em}{#1}{#2}%
}
\newcommand{\appendixsubsubentry}[2]{%
  \appendixcontentsline{3}{3.4em}{3.5em}{#1}{#2}%
}
\newcommand{\appendixunnumberedentry}[2]{%
  \par\addvspace{0.35em}
  \begingroup\bfseries
  \c@tocdepth=3\relax
  \@dottedtocline{1}{0em}{0em}{\hyperref[#1]{#2}}{\hyperref[#1]{\pageref*{#1}}}%
  \endgroup
}
\makeatother

\newcommand{\method}{HiLoRe\xspace}
\newcommand{\Hact}{\ensuremath{\mathrm{H}}\xspace}
\newcommand{\Lact}{\ensuremath{\mathrm{L}}\xspace}
\newcommand{\Ract}{\ensuremath{\mathrm{R}}\xspace}

\title{HiLoRe: What to Store, Compress, or Recompute for Efficient GRPO Training}

\author{
Xinrui Chen$^{1}$
\quad
Mengyang Li$^{2}$
\quad
Ou Wu$^{1}$\thanks{Corresponding author.}
\quad
Ji Zhang$^{3}$\\
$^{1}$Hangzhou Institute for Advanced Study,\\
University of Chinese Academy of Sciences, Hangzhou, China\\
$^{2}$Tianjin Key Laboratory of Wireless Mobile Communications and Power Transmission,\\
Tianjin Normal University, Tianjin, China\\
$^{3}$University of Southern Queensland\\
\texttt{chenxinrui25@mails.ucas.ac.cn}\\
}

\iclrfinalcopy 
\begin{document}

\maketitle
\lhead{Preprint}

\begin{abstract}
Group-relative policy optimization (GRPO) makes learner-side activations a major memory--computation bottleneck: gradient checkpointing reduces activation memory through recomputation, but fixed schedules can leave roughly 18~GB unused on a 48-GB GPU despite substantial recomputation overhead. Existing activation-management methods set state fidelity from execution cost, tensor properties, or generic compression sensitivity, without explicitly incorporating GRPO's analytic update structure into state-fidelity allocation. We formalize this dependence as policy-update exposure, linking the current GRPO loss coefficients to state-level approximation sensitivity. These coefficients are available before backward without an additional backward pass. We introduce \textbf{\method}, which allocates graph-attributed recovery units among \underline{\textbf{hi}}gh-precision storage, \underline{\textbf{lo}}w-precision compression, and deterministic \underline{\textbf{re}}computation using measured recovery utility and update-conditioned approximation risk. It combines high-precision storage and deterministic recomputation with low-precision recovery under a calibrated risk budget. Across five model--task settings with 2K responses and memory $\le1.10\times$ GC's per-GPU actor-update peak, \method's actor-update throughput gains reach 13.5\% over GC and 7.9\% over the fastest evaluated baseline, with paired mean downstream-score differences below 0.6 percentage points.
\href{https://anonymous.4open.science/r/HiLoRe}{Code}.
\end{abstract}

\section{Introduction}

GRPO is common in verifiable-reward LLM post-training~\citep{shao2024deepseekmath,song2026AdaReasoner,Mroueh2026Revisiting}, where multiple responses increase learner workload, making retained activations a memory--computation bottleneck. Gradient checkpointing reduces activation residency via recomputation; compression and hybrid recovery lower retained-state cost via approximation~\citep{chen2016training,liu2022gact,korthikanti2023reducing,chen2025adacc,Doan2026INSTANT,khalaf2026qkv,wei2026activation}. Existing methods set fidelity from execution structure, tensor properties, compression error, or generic sensitivity. Fig.~\ref{fig:teaser} shows unused memory headroom under fixed recovery (Fig.~\ref{fig:teaser}a) and update-dependent state fidelity (Fig.~\ref{fig:teaser}b). \textbf{Existing recovery policies do not explicitly connect GRPO's analytic update structure to state-level fidelity allocation.}

We formalize this update--state coupling as \textbf{policy-update exposure}. A state with little current update mass can tolerate more aggressive approximation, whereas one strongly supporting the active update requires higher fidelity. GRPO already provides exact analytic per-token update coefficients before backward, combining rollout advantages, policy ratios, clipping states, update masks, and KL regularization without an additional backward pass (Fig.~\ref{fig:teaser}c). In post-update diagnostic replays, \textbf{54.8--73.6\% of padded response slots have zero policy-gradient contribution, while the top 10\% by $|\omega_t|$ carry 52.4--66.8\% of total coefficient mass}. Mapping these coefficients to recovery states turns the optimization signal into a state-level predictor of approximation-induced gradient distortion.

This formulation yields \textbf{\method} for update-aware recovery allocation over graph-attributed states. \method distinguishes recovery actions by update effect: high-precision storage and deterministic recomputation preserve the update under execution-equivalent conditions, while low-precision recovery is selected under an explicit fidelity constraint. During the exact forward pass, \method retains a memory-bounded set of high-precision and compressed state bundles. The current update then selects available bundles for backward recovery and assigns the remaining units to deterministic recomputation. Allocating memory to recovery actions with high realized utility and admissible update risk, \method improves actor-update efficiency while preserving update fidelity.

\begin{figure}[t]
    \centering
    \includegraphics[width=\linewidth, trim=0mm 0mm 0mm 5mm, clip]{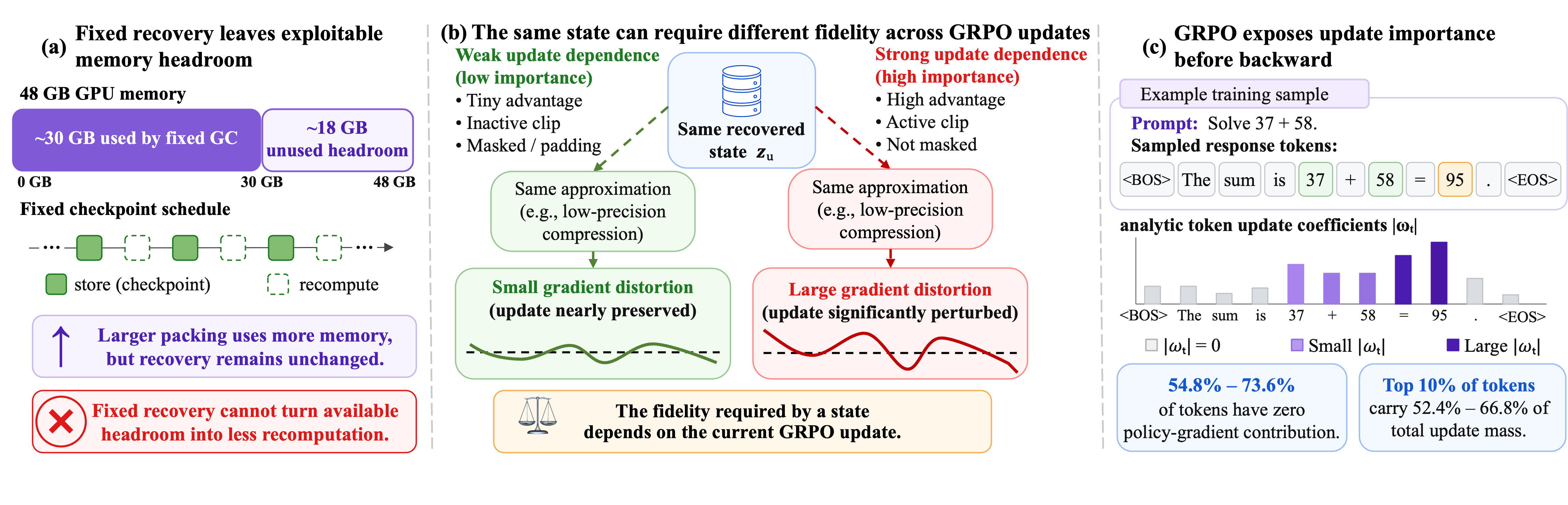}
    \vspace{-0.45in}
    \caption{Motivation of update-conditioned recovery. Fixed recovery underutilizes memory headroom, while GRPO exposes update-dependent state fidelity before backward.}
    \label{fig:teaser}
    \vspace{-0.2in}
\end{figure}

\method optimizes learner-state recovery under a fixed GRPO workload, attributing gains to recovery allocation. Fixed-workload comparisons demonstrate improvements on the memory--throughput--fidelity frontier across tasks, response lengths, attention architectures, and 3B--8B models. On-policy training evaluations show downstream quality near paired GC. Post-update diagnostic replays evaluate how well policy-update exposure predicts approximation-induced gradient distortion. Across five model--task settings spanning three models and three tasks, \method improves actor-update throughput by up to 13.47\% over GC and 7.87\% over the fastest evaluated baseline with 2K responses at $B=1.10B_{\mathrm{GC}}$. Paired mean downstream scores differ from GC by less than 0.6 percentage points.
 
Our contributions are:
\begin{itemize}[leftmargin=*,nosep]
    \item We formulate \textbf{policy-update exposure} to connect the current GRPO update to state-level recovery fidelity using analytic per-token loss coefficients.
    \item We derive an update-conditioned gradient-deviation bound for low-precision recovery, linking reconstruction error to policy-update exposure.
    \item We introduce \textbf{\method}, a graph-attributed H/L/R allocator that combines realized recovery utility with calibrated, update-conditioned approximation risk under explicit memory and fidelity constraints for a fixed GRPO workload.
    \item Actor-update throughput improves within evaluated memory budgets across workloads, lengths, architectures, and scales, with mean downstream scores near paired GC. Post-update diagnostics show improved risk prediction over evaluated static and generic sensitivity signals.
\end{itemize}

\section{Related Work}

\textbf{Training-State Recovery and Activation Compression.}
Training-memory methods reduce activation residency via checkpointing/rematerialization~\citep{chen2016training,jain2020checkmate,korthikanti2023reducing,zhao2023rockmate,li2026out}, compression~\citep{chen2021actnn,liu2022gact,woo2024alam,shamshoum-etal-2025-compact,li2026prac}, and hybrid compression--recomputation~\citep{chen2025adacc}. Recent work studies QKV states, compression criteria, layer-aware quantization, activation subspaces, and resource-aware scheduling~\citep{khalaf2026qkv,wei2026activation,lin2026agoq,sakr2024espace,Doan2026INSTANT,choudhary2026oasis,Pan2026SNIP,wu2026lazytrain}. They exploit execution structure, representations, or sensitivity; \method allocates backward-state fidelity using GRPO signals, preserving forward computation.

\textbf{Learning-Signal-Aware Training Efficiency.}
Learning signals including saliency, gradients, advantages, and entropy have been used to prioritize tokens, activations, and trajectories~\citep{simoulin-etal-2024-memory,kim-lee-2025-forward,zeng2026tokenseek,wang2025d3s,zhang2026gcpo}. In LLM RL, token and rollout selection and cost-aware optimization use such signals to select workloads~\citep{sang2026not,xu2026not,mohri2026cost}; forward-pass saliency has also guided activation selection~\citep{kim-lee-2025-forward}. \method uses GRPO update signals to set learner-state fidelity under a fixed workload, distinguishing recovery from workload selection and forward-pass saliency.

\textbf{Efficient LLM Reinforcement-Learning Post-Training.}
LLM RL systems combine faster execution, low-precision arithmetic, and learner/rollout optimizations~\citep{cui2025metrorlhf,Zhang2026FastGRPO,qiu2026fp8,huang2026qerl,li2026qurl,zhuge2026quads}. OpenRLHF~\citep{hu2025openrlhf} coordinates distributed rollout and training; Unsloth optimizes training kernels and memory management. Sparse rollout~\citep{luo-etal-2026-sparse,zhou2026sparrow}, token or rollout selection~\citep{sang2026not,xu2026not}, and KV-state reduction, reuse, or long-context scheduling~\citep{wang2026efficient,zhu2026compress,gai2026dualkv,zhou2026longstraw,li2026schedule} reduce RL pipeline costs. \method focuses on update-conditioned learner-state recovery within a fixed GRPO workload.

\begin{figure}[t]
    \centering
    \includegraphics[width=\linewidth, trim=0mm 0mm 0mm 0mm, clip]{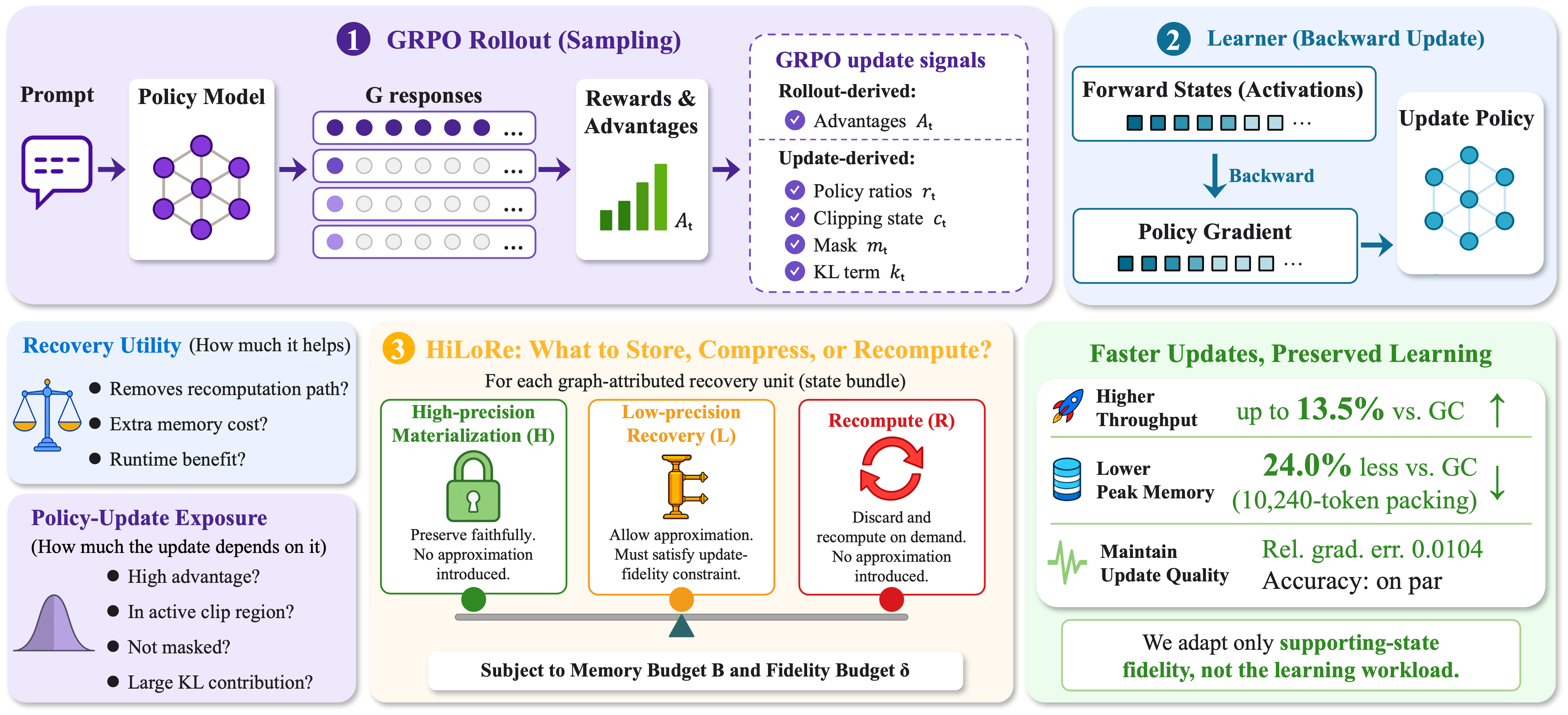}
    \vspace{-0.3in}
    \caption{\method maps GRPO-native signals and utility to fidelity-constrained H/L/R scheduling.}
    \label{fig:overview}
    \vspace{-0.15in}
\end{figure}

\section{HiLoRe}
\label{sec:method}

\method conditions learner-state fidelity on the active GRPO update and combines policy-update exposure with measured recovery utility to allocate graph-attributed units to high-precision storage, low-precision recovery, or recomputation under memory and fidelity constraints.

\subsection{Problem Setup}
\label{sec:problem}

\textbf{GRPO policy update.}
GRPO induces heterogeneous token-level update contributions through advantages, policy ratios, clipping states, and KL regularization. For each prompt $q_i$, GRPO samples $G$ responses $\{o_i^{(j)}\}_{j=1}^{G}$ and computes group-relative advantages $\widehat A_i^{(j)}$. Each response shares its advantage across constituent tokens. For token $t$, let $\widehat A_t$ denote its response advantage. Let $\theta$ and $\theta_{\mathrm{old}}$ be the current and old policy parameters, with $s_t=\log\pi_\theta(o_t\mid q,o_{<t})$, $s_t^{\mathrm{old}}=\log\pi_{\theta_{\mathrm{old}}}(o_t\mid q,o_{<t})$, $r_t=\exp(s_t-s_t^{\mathrm{old}})$, and learner update mask $m_t\in\{0,1\}$. With $N=\sum_t m_t$, the learner optimizes
\begin{equation}
\mathcal L(\theta)=
-\frac{1}{N}\sum\nolimits_t m_t
\min\!\left(r_t\widehat A_t,
\operatorname{clip}(r_t,1-\epsilon,1+\epsilon)\widehat A_t\right)
+\beta\mathcal L_{\mathrm{KL}}(\theta),
\label{eq:grpo-objective}
\end{equation}
where $\epsilon$ is the clipping radius and $\beta$ weights KL regularization against the reference policy. The learner forward pass provides the current-policy log-probabilities needed to compute the exact update coefficients in Eq.~\ref{eq:update-coefficient}, without an additional backward pass.

\textbf{Recovery decisions.}
Let $\mathcal U=\{1,\ldots,U\}$ denote verified recovery units in the learner graph. Each unit $u$ couples a forward subgraph with the complete state bundle required by backward execution and serves as a graph-valid recovery decision object, admitting actions $\mathcal A_u\subseteq\{\Hact,\Lact,\Ract\}$:
\begin{itemize}[leftmargin=1.8em,nosep,label=\ding{228}]
    \item $\Hact$ (store): retain the original state bundle in training precision;
    \item $\Lact$ (compress): retain a low-precision representation of the bundle and reconstruct it for backward;
    \item $\Ract$ (recompute): release the bundle and recompute its forward subgraph for backward.
\end{itemize}
A candidate-storage plan $\mathbf q=(q_1,\ldots,q_U)$ is fixed before each learner microbatch forward. Its retained representations restrict $\mathcal A_u$ to available recovery actions $\mathcal A_u(q_u)$: exact bundles support H/L/R, compressed bundles L/R, and checkpoints R. Current GRPO coefficients select $\mathbf a$ after forward (Appendix~\ref{app:online-scheduling}).
$\Hact$ and $\Ract$ preserve the recovered state, while $\Lact$ introduces controlled approximation. Let $d_u(a)$ denote approximation-induced update distortion:
\begin{equation}
d_u(\Hact)=d_u(\Ract)=0,
\qquad
d_u(\Lact)\ge0.
\label{eq:recovery-distortion}
\end{equation}

\textbf{Optimization objective.}
For a pre-forward candidate plan $\mathbf q$, let $T_{\mathrm{upd}}(\mathbf q,\mathbf a)$ and $M_{\mathrm{peak}}(\mathbf q,\mathbf a)$ denote actor-update time and peak device memory over candidate capture and backward recovery, and $D(\mathbf a)$ aggregate update distortion. Given memory and fidelity budgets $B,\delta$, recovery selection solves
\begin{equation}
\min_{\mathbf a}\quad T_{\mathrm{upd}}(\mathbf q,\mathbf a)\quad
\mathrm{s.t.}\quad M_{\mathrm{peak}}(\mathbf q,\mathbf a)\le B,\quad D(\mathbf a)\le\delta,\quad a_u\in\mathcal A_u(q_u)\ \forall u.
\label{eq:problem-objective}
\end{equation}
The formulation separates the two quantities governing recovery: realized recovery value and the fidelity demanded by the active update under approximation.

\subsection{Policy-Update Exposure}
\label{sec:exposure}

The GRPO objective provides the token-level quantity that weights backward contributions to the active policy update. Define the effective update coefficient of token $t$ by $\omega_t=\partial\mathcal L/\partial s_t$. Away from clipping boundaries, differentiating Eq.~\ref{eq:grpo-objective} with respect to $s_t$ gives
\begin{equation}
\omega_t=-\frac{m_t}{N}\chi_t r_t\widehat A_t+\beta\frac{\partial\mathcal L_{\mathrm{KL}}}{\partial s_t},
\qquad
\chi_t=\mathbf{1}\!\left[(\widehat A_t\ge0\land r_t\le1+\epsilon)\lor(\widehat A_t<0\land r_t\ge1-\epsilon)\right].
\label{eq:update-coefficient}
\end{equation}
At clipping boundaries, $\omega_t$ follows the learner implementation's derivative convention. The coefficient $\omega_t$ weights token $t$'s contribution to the current parameter gradient, as formalized in Proposition~\ref{prop:exact-decomposition}.
\begin{propositionbox}
\begin{proposition}[Exact decomposition of the GRPO update]
\label{prop:exact-decomposition}
Let $g_t=\nabla_\theta s_t$. With rollout quantities fixed and the token-separable sampled-token KL objective in Appendix~\ref{app:update-semantics}, Eq.~\ref{eq:grpo-objective} satisfies
\begin{equation}
\nabla_\theta\mathcal L=\sum\nolimits_t\omega_t g_t,\qquad \omega_t=\frac{\partial\mathcal L}{\partial s_t}.
\label{eq:exact-decomposition}
\end{equation}
The coefficients $\{\omega_t\}$ are exactly available from rollout data and forward-computed GRPO signals.
\end{proposition}
\end{propositionbox}
Proposition~\ref{prop:exact-decomposition} links learner-state recovery to GRPO coefficients encoding rollout advantage, policy ratio, clipping state, update mask, and per-token KL contribution. Modifying saved activations preserves these coefficients when exact forward log-probabilities and loss-side quantities remain fixed. Appendix~\ref{app:update-semantics} derives the decomposition and specifies exact H/R recovery conditions.

Let $z_u$ be unit $u$'s exact backward state and $\widehat z_u$ its low-precision reconstruction, with $\Delta z_u=\widehat z_u-z_u$. Let $\mathcal J(u)$ contain response positions whose backward paths depend on $z_u$, and $\Delta g_u$ the parameter-gradient change. Let $g_t(z)$ denote the parameter-gradient vector from backpropagation with unit upstream derivative at $s_t$, with unit $u$ consuming $z$ and other saved states exact. Thus $g_t(z_u)=\nabla_\theta s_t$. Under Theorem~\ref{thm:exposure-bound}'s smoothness assumptions, perturbing only $u$ changes the parameter gradient by
\begin{equation}
\Delta g_u=\sum\nolimits_{t\in\mathcal J(u)}\omega_tG_{t,u}\Delta z_u+r_u,
\qquad
\|r_u\|_2\le\frac{1}{2}\mathcal H_u\|\Delta z_u\|_2^2.
\label{eq:gradient-perturbation}
\end{equation}
where $G_{t,u}=\left.\partial g_t(z)/\partial z\right|_{z=z_u}$ is the state-to-gradient Jacobian and $\mathcal H_u$ the corresponding second-order sensitivity coefficient. Eq.~\ref{eq:gradient-perturbation} identifies the state-level quantity required for update-conditioned fidelity. We define the ideal \textbf{policy-update exposure} and its operational proxy as
\begin{equation}
\mathcal E_u=\sum\nolimits_{t\in\mathcal J(u)}|\omega_t|\|G_{t,u}\|_{\mathrm{op}},
\qquad
\widehat{\mathcal E}_u=\sum\nolimits_{t\in\mathcal J(u)}|\omega_t|\alpha_{t,u}.
\label{eq:update-exposure}
\end{equation}
where $\alpha_{t,u}\ge0$ is an empirical token weight within the structural support. The evaluated MLP proxy derives these weights from tokenwise activation magnitudes, as specified in Appendix~\ref{app:exposure-calibration}. Theorem~\ref{thm:exposure-bound} bounds gradient deviation using the ideal exposure $\mathcal E_u$:
\begin{equation}
\|\Delta g_u\|_2\le\mathcal E_u\|\Delta z_u\|_2+\frac{1}{2}\mathcal H_u\|\Delta z_u\|_2^2.
\label{eq:exposure-bound-main}
\end{equation}
For a fixed perturbation, Eq.~\ref{eq:exposure-bound-main} bounds gradient deviation through update-dependent $\mathcal E_u$ and a second-order residual, motivating state-specific fidelity allocation. \method computes exact update coefficients and obtains $\mathcal J(u)$ by graph reachability (Appendix~\ref{app:graph-attribution}). Appendix~\ref{app:exposure-calibration} specifies the activation-based token weights used to score the evaluated full-microbatch MLP units.

\begin{remarkbox}
\begin{remark}[Zero policy-gradient contributions]
\label{rem:near-zero-exposure}
For $\beta=0$, masked positions, zero advantages, and flat clipped policy-loss branches yield $\omega_t=0$. If coefficients vanish throughout $\mathcal J(u)$, any perturbation preserving well-defined backward execution leaves the gradient unchanged (Corollary~\ref{cor:zero-exposure}). With KL regularization, unmasked positions with zero policy-loss coefficients can retain a KL contribution; masked positions remain zero under masked KL.
\end{remark}
\end{remarkbox}

\subsection{Graph-Attributed Recovery Units}
\label{sec:recovery-units}

\method forms graph-valid recovery units from graph lineage, module ownership, and complete state bundles required by backward dependencies. Accounting for fused operations, shared inputs, checkpoint boundaries, and storage aliasing ensures executable, semantically complete decisions (Appendix~\ref{app:graph-attribution}). Candidates enter allocation as \textbf{recomputation-eliminating} units only when profiling confirms that materializing their complete bundles removes the attributed recomputation paths. Profiling establishes useful savings; current-update selection determines which bundles backward consumes. Each unit admits feasible actions $\mathcal A_u\subseteq\{\Hact,\Lact,\Ract\}$.

\Hact and \Ract introduce no approximation distortion and are evaluated by recovery utility and memory cost. \Lact introduces approximation with an update-conditioned fidelity cost, quantified in Sec.~\ref{sec:utility-risk} from exposure (Sec.~\ref{sec:exposure}). Computational value and approximation risk are assessed separately.

\subsection{Recovery Utility and Approximation Risk}
\label{sec:utility-risk}

For fixed $\mathbf q$, let $\mathbf a^R$ recompute every unit after candidate capture, and $\mathbf a^{R\rightarrow a_u}$ change only unit $u$ to an available non-R action $a$. Conditional recovery utility and marginal peak-memory cost are
\begin{equation}
\begin{aligned}
\Delta\tau_u(a\mid\mathbf q)&=T_{\mathrm{upd}}(\mathbf q,\mathbf a^R)-T_{\mathrm{upd}}(\mathbf q,\mathbf a^{R\rightarrow a_u}),\\
\Delta m_u(a\mid\mathbf q)&=\left[M_{\mathrm{peak}}(\mathbf q,\mathbf a^{R\rightarrow a_u})-M_{\mathrm{peak}}(\mathbf q,\mathbf a^R)\right]_+.
\end{aligned}
\label{eq:recovery-utility}
\end{equation}
Here $[x]_+=\max(x,0)$. Both executions share $\mathbf q$; utility includes subsequent encoding and reconstruction. Complete actor-update timing also includes capture of every candidate, even those discarded. Appendix~\ref{app:recovery-utility} relates this conditional utility to net GC savings.

To convert exposure into a fidelity cost for \Lact, \method calibrates at initialization on independent microbatches $\mathcal C$. For each \Lact-eligible unit $u$ and batch $b\in\mathcal C$, matched passes differ only in unit $u$'s low-precision recovery; other states remain exact. $z_{u,b}$ and $\widehat z_{u,b}$ denote exact and reconstructed states, and $g_{\mathrm{ex},b}$ and $g_{L,u,b}$ the exact and single-unit low-precision gradients. We measure
\begin{equation}
e_{u,b}^L=\|\widehat z_{u,b}-z_{u,b}\|_2/(\|z_{u,b}\|_2+\varepsilon),
\qquad
y_{u,b}=\sqrt{\|g_{L,u,b}-g_{\mathrm{ex},b}\|_2^2/\max\{\|g_{\mathrm{ex},b}\|_2^2,\varepsilon_{\mathrm{num}}\}}.
\label{eq:calibration-metrics}
\end{equation}
where $\varepsilon>0$ and $\varepsilon_{\mathrm{num}}>0$ stabilize reconstruction-error and gradient-error normalization, respectively, and are distinct from the GRPO clipping radius $\epsilon$. The relative-error bound in Appendix~\ref{app:exposure-risk-theory} motivates combining exposure and reconstruction error, with calibrated susceptibility capturing residual sensitivity and normalization scales. For unit $u$, let $c(u)$ denote its state/layer group and define $x_{u,b}=\widehat{\mathcal E}_{u,b}e_{u,b}^L$. We estimate susceptibility by
\begin{equation}
\kappa_c=\operatorname*{arg\,min}\nolimits_{\kappa\ge0}\sum\nolimits_{(u,b):\,c(u)=c}\left(y_{u,b}-\kappa x_{u,b}\right)^2.
\label{eq:susceptibility}
\end{equation}
Calibrated susceptibility $\kappa_{c(u)}$ captures residual sensitivity beyond the exposure proxy. Initialization fixes $\kappa_{c(u)}$ and the mean reconstruction error $\bar e_u^L=|\mathcal C|^{-1}\sum_{b\in\mathcal C}e_{u,b}^L$. During training, \method combines them with current-update exposure to compute $\psi_u=\kappa_{c(u)}\widehat{\mathcal E}_u\bar e_u^L$, the calibrated risk surrogate used as \Lact's fidelity cost. For the evaluated full-microbatch MLP units, tokenwise weights aggregate the current update coefficients into a single recovery score:
\begin{equation}
\psi_{u,b}=\kappa_{c(u)}\bar e_u^L\sum\nolimits_{t\in\mathcal T_b}|\omega_t|\alpha_{t,u,b}.
\label{eq:token-weighted-risk}
\end{equation}
Weights vary across tokens and units; identical update coefficients can yield different exposures across saved tensors. Each tensor receives one H/L/R action across all tokens. Token heterogeneity informs action scoring; selection also considers measured recovery utility and memory cost.

\subsection{Joint H/L/R Allocation}
\label{sec:joint-allocation}

For a fixed microbatch with candidate plan $\mathbf q$ and surrogate budget $\delta_b$, set $\Delta\tau_u(\Ract\mid\mathbf q)=\Delta m_u(\Ract\mid\mathbf q)=0$ and define $\widehat d_u(\Hact)=\widehat d_u(\Ract)=0$ and $\widehat d_u(\Lact)=\psi_u$. The first-order triangle bound in Lemma~\ref{lem:joint-risk} motivates aggregating these calibrated risks as $\widehat D(\mathbf a)=\sum_{u=1}^{U}\widehat d_u(a_u)$. Complete-update validation checks the joint gradient error (Appendix~\ref{app:allocation-validation}). We select recovery actions $a_u\in\mathcal A_u(q_u)$ by solving:
\begin{equation}
\max_{\mathbf a}\ \sum\nolimits_{u=1}^{U}\Delta\tau_u(a_u\mid\mathbf q)
\quad\mathrm{s.t.}
\sum\nolimits_{u=1}^{U}\Delta m_u(a_u\mid\mathbf q)\le B-M_{\mathrm{peak}}(\mathbf q,\mathbf a^R),\quad
\widehat D(\mathbf a)\le\delta_b.
\label{eq:joint-allocation}
\end{equation}

\begin{wrapfigure}{r}{0.45\columnwidth}
\vspace{-0.32in}
\begin{minipage}{\linewidth}
\begin{algorithm}[H]
\caption{\method Recovery Allocation}
\label{alg:hilore}
\begingroup
\footnotesize
\linespread{1}\selectfont
\algrenewcommand\algorithmicindent{0.95em}
\begin{algorithmic}[1]
\Require Profiled plans $\mathcal Q$, calibrated $\kappa,\bar e^L$, budgets $B,\delta$
\For{each actor update}
    \State $\textsc{ZeroGrad}()$; form microbatches $\{b\}$
    \State Assign $\{\delta_b\}$ with $\sum_b\delta_b\le\delta$
    \For{each microbatch $b$}
        \State $\mathbf q_b\gets\textsc{Plan}(\mathcal Q,\operatorname{shape}(b),B)$
        \State $(\mathcal L_b,s_b,\alpha_b)\gets\textsc{Forward}(b,\mathbf q_b)$
        \State $\omega_b\gets\textsc{Coefficients}(s_b,b)$
        \State $\widehat{\mathcal E}_b\gets\textsc{Exposure}(\omega_b,\alpha_b)$
        \State $\psi_b\gets\kappa\widehat{\mathcal E}_b\bar e^L$
        \State $\mathbf a_b\gets\textsc{Allocate}(\mathbf q_b,\psi_b;B,\delta_b)$
        \State $\textsc{Backward}(\mathcal L_b,\mathbf q_b,\mathbf a_b)$
    \EndFor
    \State $\textsc{OptimizerStep}()$
\EndFor
\end{algorithmic}
\endgroup
\end{algorithm}
\end{minipage}
\vspace{-0.2in}
\end{wrapfigure}

Eq.~\ref{eq:joint-allocation} uses additive utility, memory, and risk estimates. Conservatively discretizing memory and risk yields a two-constraint multiple-choice knapsack, which dynamic programming solves exactly on the grid (Appendix~\ref{app:allocation-validation}). Validation measures complete capture-and-recovery peak memory and accumulated gradient error, capturing shared recomputation, kernel interactions, and multi-unit low-precision effects; non-\Ract actions with non-positive leave-one-out runtime utility are removed. Algorithm~\ref{alg:hilore} summarizes scheduling. \method reuses validated plans, profiled utilities, and initialization-calibrated statistics. Each microbatch refreshes GRPO coefficients and risk scores after forward, then selects actions over available representations.

\textbf{Complexity.}
Graph attribution, utility profiling, and perturbation calibration are amortized. Calibrating $U$ units over $C$ microbatches requires $O(CU)$ evaluations. $V$ counts activation elements reduced for token weights; $I$, token-to-unit incidences. Weight collection costs $O(V)$; coefficient and score refresh, $O(|\mathcal T|+I+U)$; weight storage, $O(I)$. Allocation over $K_B$ memory and $K_\delta$ fidelity states costs $O(UK_BK_\delta)$. Actor-update measurements include collection, refresh, and allocation.

\section{Experiments}
\label{sec:experiments}

We evaluate whether \method converts available memory into faster GRPO updates without sacrificing update fidelity, whether policy-update exposure reliably identifies approximation-sensitive states, and whether these gains persist across budgets, workloads, response lengths, and models.

\subsection{Experimental Setup}
\label{sec:experimental-setup}

\textbf{Models \& Datasets.}
We use Qwen2.5-3B-Instruct~\citep{yang2024qwen25} as primary and Phi-3.5-mini-instruct~\citep{abdin2024phi3} and Llama-3.1-8B-Instruct~\citep{grattafiori2024llama3} for transfer. GRPO trains on DeepMath10K~\citep{he2025deepmath}, TACO-Verified~\citep{li2023taco}, and Logic-RL Knights-and-Knaves~\citep{xie2025logicrl}, evaluated respectively on MATH500/GSM8K~\citep{hendrycks2021math,cobbe2021training}, LiveCodeBench/MBPP+~\citep{jain2024livecodebench,liu2023evalplus}, and held-out Knights-and-Knaves/ZebraLogic~\citep{xie2025logicrl,lin2025zebralogic}.

\textbf{Baselines.}
Baselines cover GC / All-R~\citep{chen2016training} and No-GC / All-H endpoints, rematerialization (SAC, Rockmate~\citep{zhao2023rockmate}), compression (GACT~\citep{liu2022gact}, ALAM~\citep{woo2024alam}, CompAct~\citep{shamshoum-etal-2025-compact}, PRAC~\citep{li2026prac}, AGoQ~\citep{lin2026agoq}, INSTANT~\citep{Doan2026INSTANT}), and hybrid compression/recomputation (Adacc~\citep{chen2025adacc}). We compare methods under shared execution and GRPO replays, measuring memory--throughput--fidelity trade-offs. Configurations and tuning are detailed in Appendix~\ref{app:experimental-setup}.

\textbf{Protocol \& Metrics.} Defaults: rank-16 LoRA on four L40S 48-GB GPUs. Training accumulates gradients across all microbatches before one optimizer step per on-policy rollout batch. Matched held-out replays at the pre-update learner state measure actor-update efficiency and gradient fidelity; efficiency timing includes all online method overhead. Post-update diagnostic replays evaluate risk prediction with fixed old-policy likelihoods (Appendix~\ref{app:exposure-details}). These diagnostics preserve the training trajectory and are timed separately. Memory is normalized by GC's per-GPU actor-update peak, $B_{\mathrm{GC}}$. Configuration selection favors throughput subject to shared memory budgets, a mean relative gradient-error tolerance $\epsilon_g=0.015$, and a terminal validation score at most one percentage point below GC. Five paired timing repetitions yield gain SDs; downstream scores report means and SDs across three paired on-policy training seeds. Appendix~\ref{app:experimental-setup} details protocols and metric definitions.

\subsection{Main Results}
\label{sec:canonical-results}

\begin{table*}[t]
\centering
\setlength{\aboverulesep}{1.5pt}
\setlength{\belowrulesep}{1.5pt}
\setlength{\fboxsep}{0.5pt}
\caption{Actor-update efficiency, gradient fidelity, and downstream quality on DeepMath10K with 2K responses. Red and orange indicate the \resultbest{best} and \colorbox{orange!15}{second-best} values within each model.}
\label{tab:main-results}
\scriptsize
\setlength{\tabcolsep}{3pt}
\renewcommand{\arraystretch}{0.75}
\resizebox{\linewidth}{!}{%
\begin{tabular}{@{}c|c|ccc|c|cc@{}}
\toprule
\multirow{2}{*}{\textbf{Model}} & \multirow{2}{*}{\textbf{Method}} & \multicolumn{3}{c|}{\textbf{Efficiency}} & \multicolumn{1}{c|}{\textbf{Gradient Fidelity}} & \multicolumn{2}{c@{}}{\textbf{Training Quality}} \\
\cmidrule(lr){3-5}\cmidrule(lr){6-6}\cmidrule(l){7-8}
& & Peak MiB $\downarrow$ & Tok./s $\uparrow$ & Gain (\%) $\uparrow$ & Grad. err. $\downarrow$ & MATH500 $\uparrow$ & GSM8K $\uparrow$ \\
\midrule
\multirow{8}{*}{\rotatebox[origin=c]{90}{\textbf{Qwen2.5-3B\quad}}}
& GC / All-R\venuetag{arXiv'16} & 28,874 & 2,455.72 & -- & \resultbest{0.0000} & \resultsecond{60.20_{\pm1.00}} & \resultsecond{77.33_{\pm0.15}} \\
& Rockmate\venuetag{ICML'23} & 29,992 & 2,066.73 & $-15.84_{\pm0.79}$ & \resultsecond{0.0031} & \resultsecond{60.20_{\pm0.80}} & $76.90_{\pm0.04}$ \\
& ALAM\venuetag{ICLR'24} & \resultbest{28{,}595} & 2,381.56 & $-3.02_{\pm0.57}$ & 0.0131 & $59.13_{\pm0.50}$ & $76.57_{\pm0.23}$ \\
& Adacc\venuetag{arXiv'25} & 31,400 & \resultsecond{2{,}579.00} & \resultsecond{+5.02_{\pm0.49}} & 0.0107 & $60.00_{\pm0.40}$ & $77.08_{\pm0.12}$ \\
& INSTANT\venuetag{ICLR'26} & 29,236 & 2,515.39 & $+2.43_{\pm0.47}$ & 0.0105 & $59.80_{\pm0.40}$ & $77.00_{\pm0.12}$ \\
& AGoQ\venuetag{ICML'26} & \resultsecond{28{,}776} & 2,533.81 & $+3.18_{\pm0.56}$ & 0.0110 & $59.47_{\pm0.42}$ & $76.70_{\pm0.19}$ \\
& PRAC\venuetag{ICML'26} & 29,455 & 2,523.50 & $+2.76_{\pm0.55}$ & 0.0112 & $59.67_{\pm0.31}$ & $76.88_{\pm0.15}$ \\
& \textbf{\method (Ours)} & 31,649 & \resultbest{2{,}705.22} & \resultbest{+10.16_{\pm0.44}} & 0.0104 & \resultbest{60.73_{\pm0.31}} & \resultbest{77.84_{\pm0.12}} \\
\midrule
\multirow{8}{*}{\rotatebox[origin=c]{90}{\textbf{Phi-3.5-mini}\quad}}
& GC / All-R\venuetag{arXiv'16} & 17,607 & 2,610.33 & -- & \resultbest{0.0000} & $44.00_{\pm0.60}$ & $85.97_{\pm0.23}$ \\
& Rockmate\venuetag{ICML'23} & 18,290 & 2,325.54 & $-10.91_{\pm0.61}$ & \resultsecond{0.0031} & $43.93_{\pm0.50}$ & $85.85_{\pm0.23}$ \\
& ALAM\venuetag{ICLR'24} & \resultbest{17{,}476} & 2,562.30 & $-1.84_{\pm0.54}$ & 0.0125 & $43.40_{\pm0.80}$ & $85.37_{\pm0.30}$ \\
& Adacc\venuetag{arXiv'25} & 18,963 & \resultsecond{2{,}725.71} & \resultsecond{+4.42_{\pm0.51}} & 0.0107 & \resultsecond{44.07_{\pm0.50}} & $85.92_{\pm0.19}$ \\
& INSTANT\venuetag{ICLR'26} & 17,818 & 2,716.57 & $+4.07_{\pm0.50}$ & 0.0102 & $43.80_{\pm0.60}$ & $85.82_{\pm0.23}$ \\
& AGoQ\venuetag{ICML'26} & \resultsecond{17{,}519} & 2,721.53 & $+4.26_{\pm0.56}$ & 0.0108 & $43.40_{\pm0.60}$ & \resultbest{86.25_{\pm0.27}} \\
& PRAC\venuetag{ICML'26} & 18,032 & 2,710.57 & $+3.84_{\pm0.52}$ & 0.0109 & $43.80_{\pm0.60}$ & $85.75_{\pm0.30}$ \\
& \textbf{\method (Ours)} & 19,322 & \resultbest{2{,}940.28} & \resultbest{+12.64_{\pm0.47}} & 0.0091 & \resultbest{44.53_{\pm0.50}} & \resultsecond{86.10_{\pm0.19}} \\
\midrule
\multirow{8}{*}{\rotatebox[origin=c]{90}{\textbf{Llama-3.1-8B}\quad}}
& GC / All-R\venuetag{arXiv'16} & 39,601 & 1,494.23 & -- & \resultbest{0.0000} & $55.00_{\pm0.60}$ & \resultsecond{82.87_{\pm0.23}} \\
& Rockmate\venuetag{ICML'23} & 41,000 & 1,389.63 & $-7.00_{\pm0.70}$ & \resultsecond{0.0031} & \resultbest{55.73_{\pm0.61}} & $82.71_{\pm0.23}$ \\
& ALAM\venuetag{ICLR'24} & \resultbest{39{,}214} & 1,482.87 & $-0.76_{\pm0.52}$ & 0.0128 & $54.33_{\pm0.61}$ & $82.41_{\pm0.30}$ \\
& Adacc\venuetag{arXiv'25} & 42,900 & \resultsecond{1{,}589.56} & \resultsecond{+6.38_{\pm0.51}} & 0.0107 & $54.93_{\pm0.50}$ & $82.79_{\pm0.23}$ \\
& INSTANT\venuetag{ICLR'26} & 40,050 & 1,565.36 & $+4.76_{\pm0.50}$ & 0.0105 & $54.80_{\pm0.60}$ & $82.71_{\pm0.23}$ \\
& AGoQ\venuetag{ICML'26} & \resultsecond{39{,}463} & 1,517.84 & $+1.58_{\pm0.56}$ & 0.0110 & $54.60_{\pm0.60}$ & $82.56_{\pm0.23}$ \\
& PRAC\venuetag{ICML'26} & 40,712 & 1,514.55 & $+1.36_{\pm0.50}$ & 0.0114 & $54.73_{\pm0.50}$ & $82.64_{\pm0.23}$ \\
& \textbf{\method (Ours)} & 43,403 & \resultbest{1{,}695.50} & \resultbest{+13.47_{\pm0.48}} & 0.0100 & \resultsecond{55.53_{\pm0.31}} & \resultbest{83.09_{\pm0.23}} \\
\bottomrule
\end{tabular}}
\vspace{-0.2in}
\end{table*}

Table~\ref{tab:main-results} compares actor-update efficiency, gradient fidelity, and training quality on DeepMath10K with 2K responses. Methods share a per-model memory ceiling of $B=1.10B_{\mathrm{GC}}$.

\begin{wrapfigure}{r}{0.4\linewidth}
\centering
\vspace{-0.16in}
\includegraphics[width=\linewidth]{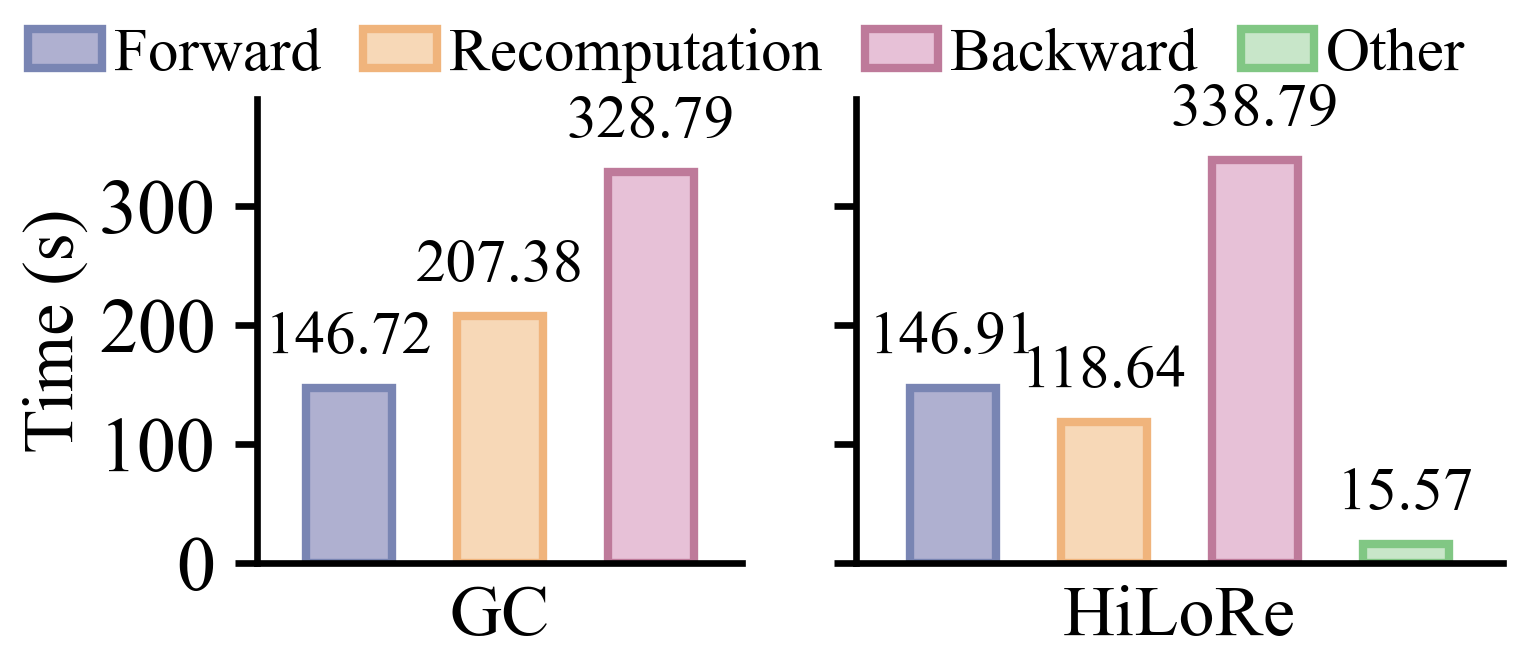}
\vspace{-0.37in}
\caption{Actor-update time breakdown.}
\label{fig:runtime-breakdown}
\vspace{-0.3in}
\end{wrapfigure}

\textbf{Higher Throughput Within Budget.} The comparison covers GQA models (Qwen2.5-3B and Llama-3.1-8B) and an MHA model (Phi-3.5-mini). On Qwen2.5-3B, \method achieves the highest throughput among compared configurations, exceeding GC and Adacc by 10.16\% and 4.89\%, respectively. This advantage also holds against SAC, GACT, and CompAct (Appendix~\ref{app:full-baseline-comparison}).

\textbf{Reduced Recomputation With High Gradient Fidelity.} Fig.~\ref{fig:runtime-breakdown} shows \method reduces recomputation time on Qwen2.5-3B from 207.38 to 118.64 seconds (42.79\%). Other (15.57 seconds) comprises encoding, reconstruction, exposure refresh, and allocation. Total component time decreases by 9.22\% (Table~\ref{tab:overhead}). Mean held-out relative gradient error is 0.0104, below the 0.015 threshold.

\begin{figure*}[h]
\centering
\begin{minipage}[t]{0.52\textwidth}
\vspace{0pt}
\centering
\captionof{table}{Cross-task efficiency and training quality.}
\label{tab:cross-task-system}
\label{tab:cross-task-timing}
\scriptsize
\renewcommand{\arraystretch}{1}
\setlength{\tabcolsep}{2.5pt}
\resizebox{\linewidth}{!}{%
\begin{tabular}{@{}lrrcc@{}}
\toprule
Method & Peak MiB $\downarrow$ & Gain (\%) $\uparrow$ & Metric 1 $\uparrow$ & Metric 2 $\uparrow$ \\
\midrule
\rowcolor{qualitypink}
\multicolumn{5}{@{}c@{}}{\textbf{TACO-Verified: LiveCodeBench / MBPP+}} \\
GC & 29,578 & -- & $23.38_{\pm0.30}$ & $61.20_{\pm0.40}$ \\
AGoQ & 29,400 & $-0.42_{\pm0.58}$ & $22.91_{\pm0.29}$ & $60.76_{\pm0.40}$ \\
Adacc & 32,080 & $+4.63_{\pm0.52}$ & $23.19_{\pm0.33}$ & $61.02_{\pm0.40}$ \\
\rowcolor{qualitygreen}
\textbf{\method} & 32,394 & $\mathbf{+9.00_{\pm0.47}}$ & $23.32_{\pm0.41}$ & $61.29_{\pm0.40}$ \\
\midrule
\rowcolor{qualityblue}
\multicolumn{5}{@{}c@{}}{\textbf{Logic-RL K\&K: held-out K\&K / ZebraLogic}} \\
GC & 29,411 & -- & $39.07_{\pm0.46}$ & $14.83_{\pm0.31}$ \\
AGoQ & 29,246 & $-0.88_{\pm0.54}$ & $38.60_{\pm0.40}$ & $14.47_{\pm0.25}$ \\
Adacc & 31,898 & $+5.37_{\pm0.50}$ & $38.80_{\pm0.35}$ & $14.73_{\pm0.25}$ \\
\rowcolor{qualitygreen}
\textbf{\method} & 32,186 & $\mathbf{+8.64_{\pm0.45}}$ & $39.13_{\pm0.31}$ & $14.93_{\pm0.21}$ \\
\bottomrule
\end{tabular}
}
\end{minipage}
\begin{minipage}[t]{0.43\textwidth}
\vspace{0.15in}
\centering
\includegraphics[width=\linewidth]{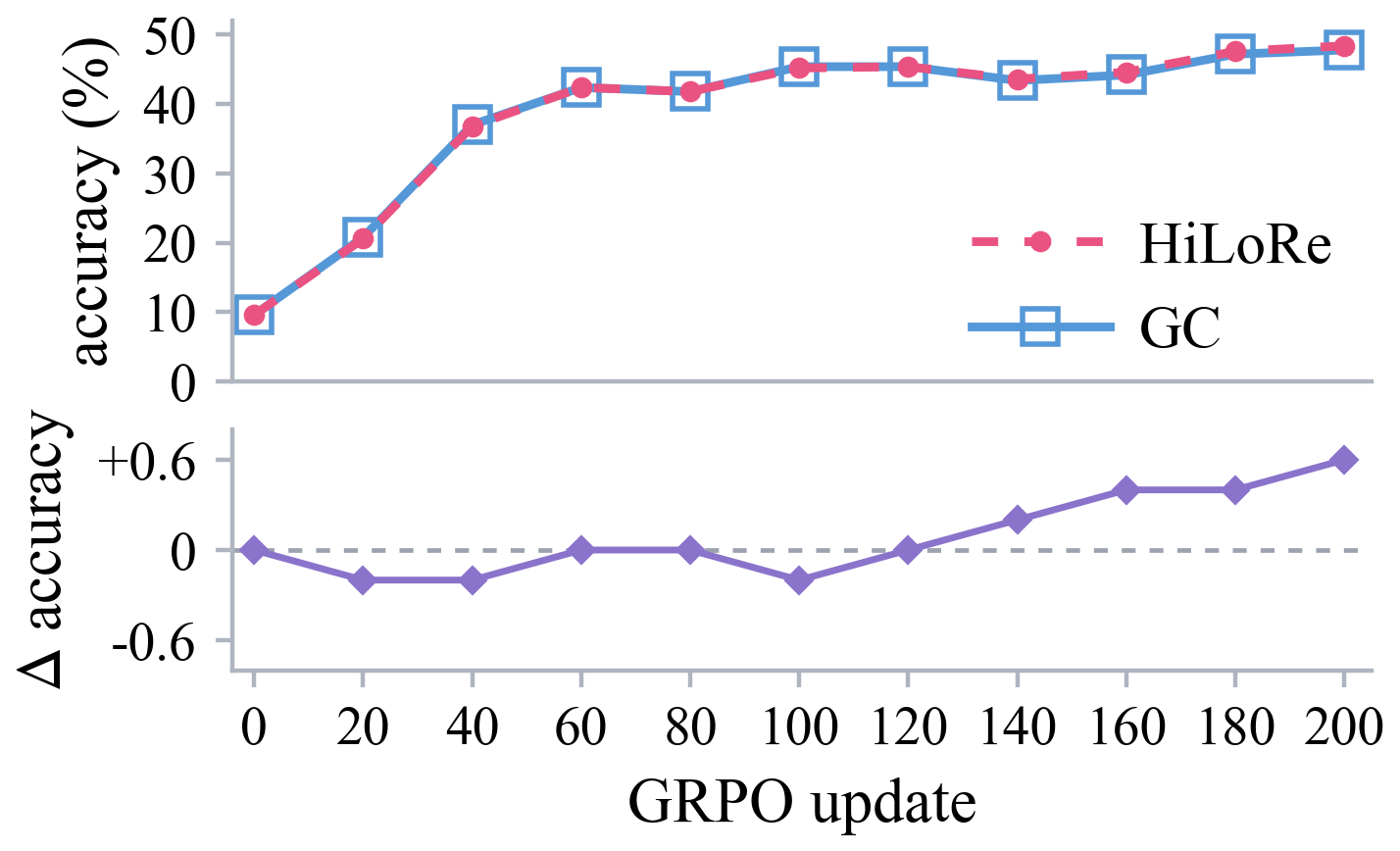}
\vspace{-0.35in}
\caption{MATH500 accuracy and \method--GC gap during GRPO training.}
\label{fig:dsr-learning}
\end{minipage}
\vspace{-0.1in}
\end{figure*}

\textbf{Generalization Across Tasks, Lengths, and Models.} At $B=1.10B_{\mathrm{GC}}$, \method leads evaluated configurations in throughput on TACO-Verified and Logic-RL K\&K, exceeding GC by 9.00\% and 8.64\%, respectively (Table~\ref{tab:cross-task-system}). Paired mean downstream-score differences from GC remain within 0.10 percentage points across four benchmarks. Across 1K/2K/4K responses, Qwen/Phi throughput gains span 6.13--13.57\%/4.73--12.64\% (Appendix~\ref{app:scaling-details}). At 4K and $B=1.05B_{\mathrm{GC}}$, three 3B--8B GQA/MHA models gain 8.39--10.18\% with 0.0090--0.0100 mean gradient errors (Fig.~\ref{fig:length-model-transfer}).

\textbf{Full-Parameter Transfer.}
We evaluate GC and \method using Llama-3.2-3B-Instruct on the 1,209-example DSR-subset with full-parameter GRPO, 3K responses, and seed 1234. Over 200 updates, GC improves MATH500 accuracy from 9.60\% to 47.80\%, while \method reaches 48.40\%. Across evaluated checkpoints, \method--GC differences span $-0.20$ to $+0.60$ percentage points (Fig.~\ref{fig:dsr-learning}), supporting learning quality close to GC. An efficiency evaluation on an A100 80GB jointly varies microbatch size and recovery allocation under budgets of $1.00$--$1.20B_{\mathrm{GC}}$ (Fig.~\ref{fig:full-parameter-runtime}). With additional memory headroom, \method achieves higher actor-update throughput than GC with budget-specific microbatch tuning, extending memory--throughput benefits to full-parameter training.

\subsection{Memory--Throughput Frontier}
\label{sec:frontier}

We sweep $B/B_{\mathrm{GC}}\in\{1.00,1.05,1.10,1.15,1.20\}$ on Qwen2.5-3B / DeepMath10K with 2K responses, $B_{\mathrm{GC}}=28{,}874$ MiB, and 4096-token packing; budgets cap measured peak memory. \circnum{1} \textbf{\method leads mean throughput across budgets.} Fig.~\ref{fig:frontier}(a) shows gains over GC rising from 3.47\% to 13.97\%, with margins over Adacc widening from 2.64 to 6.56 percentage points. \circnum{2} \textbf{Recovery shifts toward retention with errors below threshold.} In Fig.~\ref{fig:frontier}(b,c), mean held-out gradient error rises from 0.0068 to 0.0130 as H increases from 0\% to 25.0\% and R falls from 91.7\% to 61.1\%; L peaks at 16.7\%, declining to 13.9\%. Maximum observed error, 0.0141, remains below the 0.015 development threshold (Appendix~\ref{app:frontier-details}). At GC's budget, 8.3\% of units use L (28,843-MiB peak). \circnum{3} \textbf{Adaptive recovery increases throughput using headroom.} GC leaves 17,194 MiB unused on the 46,068-MiB device (Fig.~\ref{fig:teaser}a). Table~\ref{tab:headroom-control} compares increased packing with adaptive recovery. On Qwen, \method gains 13.97\% using 10,892 MiB less than GC with 10240-token packing (+10.80\%). On Llama, GC gains 4.00\% with 6144-token packing; 7168-token packing triggers OOM. With 4096-token packing, \method gains 13.47\% at $1.096B_{\mathrm{GC}}$, below GC's peak with 6144-token packing ($1.139B_{\mathrm{GC}}$). Allocation exploits headroom without increasing packing.

\begin{figure}[t]
\centering
\includegraphics[width=\linewidth]{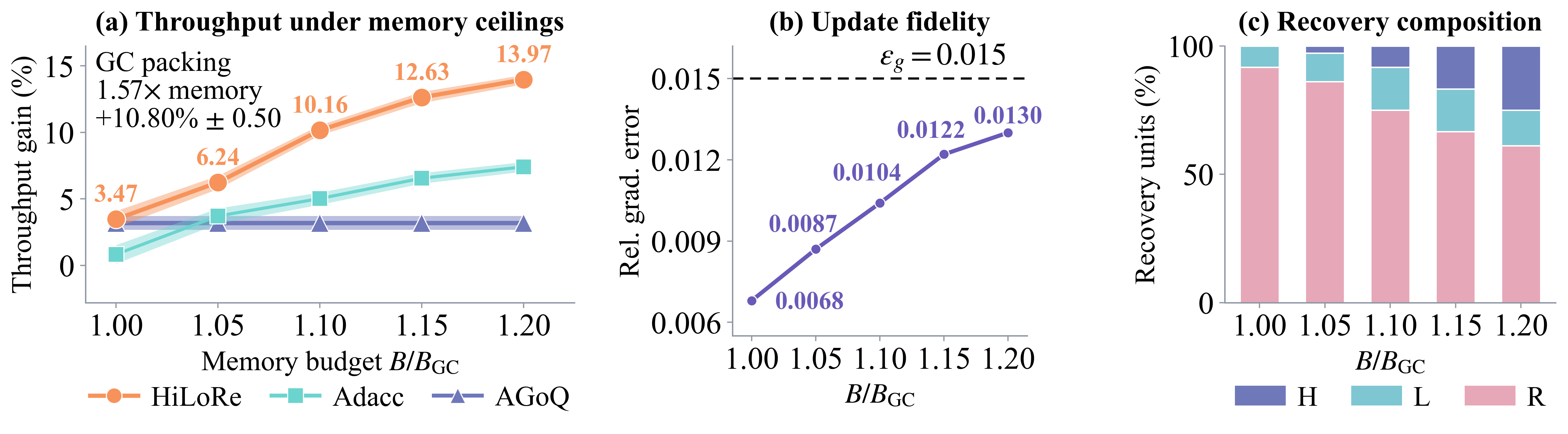}
\vspace{-0.4in}
\caption{Memory--throughput frontier, update fidelity, and recovery allocation.}
\label{fig:frontier}
\vspace{-0.1in}
\end{figure}

\begin{table*}[t]
\centering
\begin{minipage}[t]{0.5\textwidth}
\vspace{0pt}
\centering
\setlength{\abovecaptionskip}{2pt}
\setlength{\belowcaptionskip}{3pt}
\caption{Packing versus adaptive recovery.}
\label{tab:headroom-control}
\scriptsize
\setlength{\tabcolsep}{2pt}
\setlength{\aboverulesep}{1pt}
\setlength{\belowrulesep}{1pt}
\renewcommand{\arraystretch}{0.94}
\resizebox{\linewidth}{!}{%
\begin{tabular}{@{}clccc@{}}
\toprule
Model & Method & Pack & $M/B_{\mathrm{GC}}$ & Gain (\%) \\
\midrule
\multirow{4}{*}{Qwen2.5-3B} & GC & 4096 & 1.000 & -- \\
& \method$_{1.10}$ & 4096 & 1.096 & $+10.16_{\pm0.44}$ \\
& GC & 10240 & 1.574 & $+10.80_{\pm0.50}$ \\
& \method$_{1.20}$ & 4096 & 1.197 & $+13.97_{\pm0.47}$ \\
\midrule
\multirow{5}{*}{Llama-3.1-8B} & GC & 4096 & 1.000 & -- \\
& \method$_{1.10}$ & 4096 & 1.096 & $+13.47_{\pm0.48}$ \\
& GC & 6144 & 1.139 & $+4.00_{\pm0.55}$ \\
& GC & 7168 & OOM & OOM \\
& \method$_{1.15}$ & 4096 & 1.147 & $+16.24_{\pm0.50}$ \\
\bottomrule
\end{tabular}%
}
\end{minipage}
\hfill
\begin{minipage}[t]{0.47\textwidth}
\vspace{0pt}
\centering
\setlength{\abovecaptionskip}{2pt}
\setlength{\belowcaptionskip}{3pt}
\caption{Recovery-unit and allocation ablations.}
\label{tab:recovery-ablation}
\scriptsize
\setlength{\tabcolsep}{2pt}
\setlength{\aboverulesep}{1pt}
\setlength{\belowrulesep}{1pt}
\renewcommand{\arraystretch}{0.91}
\resizebox{\linewidth}{!}{%
\begin{tabular}{@{}lrrr@{}}
\toprule
Variant & Peak$/B_{\mathrm{GC}}$ & Gain (\%) $\uparrow$ & Grad. err. $\downarrow$ \\
\midrule
\rowcolor{qualitygreen}
\textbf{Full \method} & 1.096 & $\mathbf{+10.16_{\pm0.44}}$ & 0.0104 \\
\quad w/o L & 1.100 & $+6.14_{\pm0.48}$ & 0.0030 \\
\quad Static risk & 1.098 & $+8.31_{\pm0.51}$ & 0.0128 \\
\midrule
\quad Block units & 1.075 & $+2.92_{\pm0.52}$ & 0.0031 \\
\quad Random H & 1.088 & $+3.61_{\pm0.78}$ & 0.0030 \\
\quad Size-based H & 1.093 & $+4.17_{\pm0.57}$ & 0.0031 \\
\quad Uniform H & 1.097 & $+4.84_{\pm0.54}$ & 0.0030 \\
\midrule
All-R (GC) & 1.000 & -- & 0.0000 \\
All-H & OOM & OOM & OOM \\
\bottomrule
\end{tabular}
}
\end{minipage}
\vspace{-0.2in}
\end{table*}

\subsection{Risk Prediction in Post-Update Replays}
\label{sec:exposure-validation}

\begin{figure*}[t]
    \centering
    \includegraphics[width=\textwidth]{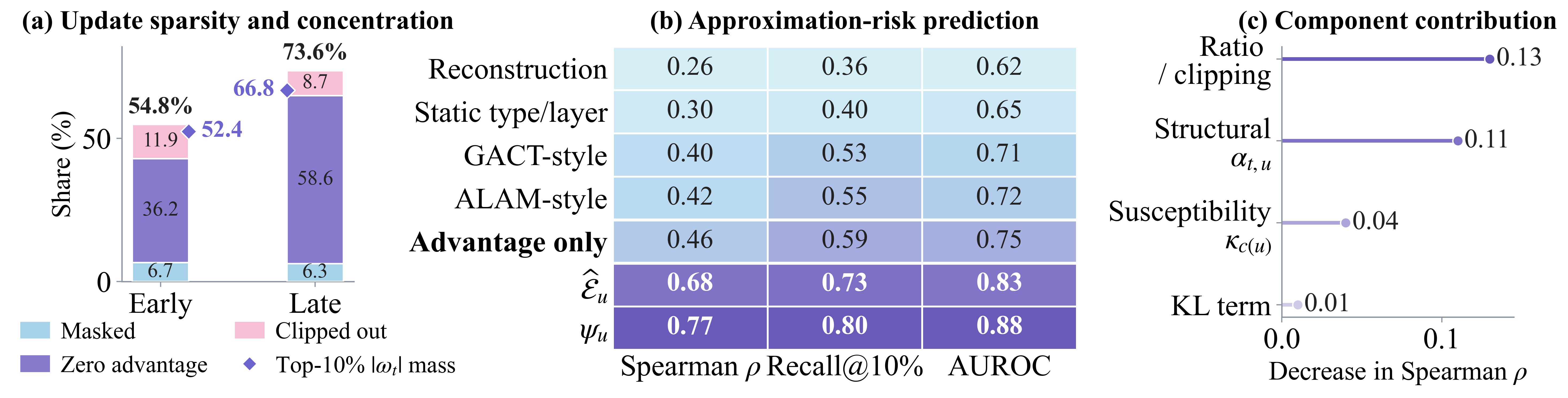}
    \vspace{-0.38in}
    \caption{Update sparsity, risk prediction, and component ablations in post-update diagnostic replays.}
    \label{fig:exposure-validation}
    \vspace{-0.1in}
\end{figure*}

Post-update diagnostic replays evaluate policy-update exposure on Qwen2.5-3B. Replays evaluate rollout batches after the reference learner's single optimizer step, holding old-policy likelihoods fixed. Predictors calibrated on disjoint records use identical held-out single-unit L interventions and relative gradient-error targets. Appendix~\ref{app:exposure-details} details the protocol. \circnum{1} \textbf{Concentrated update contributions.} On DeepMath10K, Fig.~\ref{fig:exposure-validation}(a) reports statistics over padded response slots. The fraction with zero policy-gradient contribution rises from 54.8\% to 73.6\% between replays for updates 9--10 and 11--12. The top 10\% of slots by $|\omega_t|$ carry 52.4--66.8\% of total coefficient mass. \circnum{2} \textbf{Exposure improves risk prediction.} Across DeepMath10K, TACO-Verified, and Logic-RL, Fig.~\ref{fig:exposure-validation}(b) shows macro-averaged Spearman correlations of 0.46 for advantage-only exposure, 0.68 for complete exposure, and 0.77 for $\psi_u$. Incorporating calibrated susceptibility and reconstruction distortion yields the strongest predictor, with 0.80 Top-10\% recall and 0.88 violation AUROC. \circnum{3} \textbf{Contributions to prediction.} On DeepMath10K, removing ratio/clipping information or activation weighting lowers Spearman from 0.78 by 0.13 and 0.11, respectively (Fig.~\ref{fig:exposure-validation}c). Removing susceptibility or the KL contribution lowers it by 0.04 and 0.01. The update-mask control has zero effect under response-validity masking. Appendix~\ref{app:exposure-details} reports recall and AUROC changes.

\vspace{-0.05in}
\subsection{Ablations, Sensitivity, and Operational Analysis}
\label{sec:ablations}

We evaluate recovery design and operating parameters on Qwen2.5-3B / DeepMath10K fixed-workload replay under $B=1.10B_{\mathrm{GC}}$. Defaults use $\delta^\star$, 16 calibration batches, and per-microbatch exposure refresh; gains are relative to paired All-R and include all online recovery costs.

\textbf{Recovery design and update conditioning contribute
distinct gains.}
Table~\ref{tab:recovery-ablation} shows a 6.14\% throughput gain for graph-valid H/R recovery with measured utilities, compared with 2.92--4.84\% for block-level or heuristic alternatives. Static-risk H/L/R achieves an 8.31\% gain over GC, and full \method achieves 10.16\%. Relative to static risk, current-update conditioning increases throughput by 1.71\% and reduces mean relative gradient error from 0.0128 to 0.0104, an 18.75\% reduction. This comparison isolates the incremental benefit of update conditioning within the recovery system.

\begin{wraptable}{r}{0.44\textwidth}
\vspace{-0.3in}
\centering
\caption{Exposure and recovery allocation.}
\label{tab:update-regime-main}
\footnotesize
\renewcommand{\arraystretch}{1}
\setlength{\tabcolsep}{3pt}
\resizebox{0.9\linewidth}{!}{%
\begin{tabular}{@{}lrrr@{}}
\toprule
Update regime & Norm. exposure & H (\%) & L (\%) \\
\midrule
Low $|\widehat A_t|$ & 0.48 & 4.8 & 18.5 \\
High $|\widehat A_t|$ & 1.71 & 14.7 & 9.4 \\
\midrule
Clipped & 0.27 & 2.9 & 21.3 \\
Active unclipped & 1.43 & 12.6 & 11.2 \\
\bottomrule
\end{tabular}}
\vspace{-0.2in}
\end{wraptable}

\textbf{Recovery allocation in diagnostic replays.} Table~\ref{tab:update-regime-main} reports held-out post-update replays where exposure and recovery actions are recomputed at the diagnostic state. H/L fractions follow Appendix~\ref{app:ablation-robustness}'s attribution procedure. H allocation favors high- over low-advantage regions (14.7\% vs 4.8\%), and L allocation favors clipped over active-unclipped regions (21.3\% vs 11.2\%). Isolated L errors average 0.0051 and 0.0115 for clipped and active-unclipped groups.

\textbf{Sensitivity supports default fidelity budget, calibration size, and refresh interval.} Fig.~\ref{fig:hyperparameter-sensitivity} reports 10.16\% higher throughput at $\delta^\star$, with mean gradient error 0.0104 within the 0.015 tolerance. At $2\delta^\star$, gain reaches 10.94\% and mean error 0.0158, exceeding tolerance. At $4\delta^\star$, gain reaches 11.28\% (mean error 0.0181). Calibration plateaus at 16 microbatches: 10.16\% versus 10.18\% with 32. Per-microbatch refresh gains 10.16\%, versus 10.04\% every four microbatches and 8.31\% with static risk. Stars mark defaults; throughput-gain error bars show sample SDs across five paired repetitions.

\begin{figure*}[t]
    \centering
    \includegraphics[width=\textwidth]{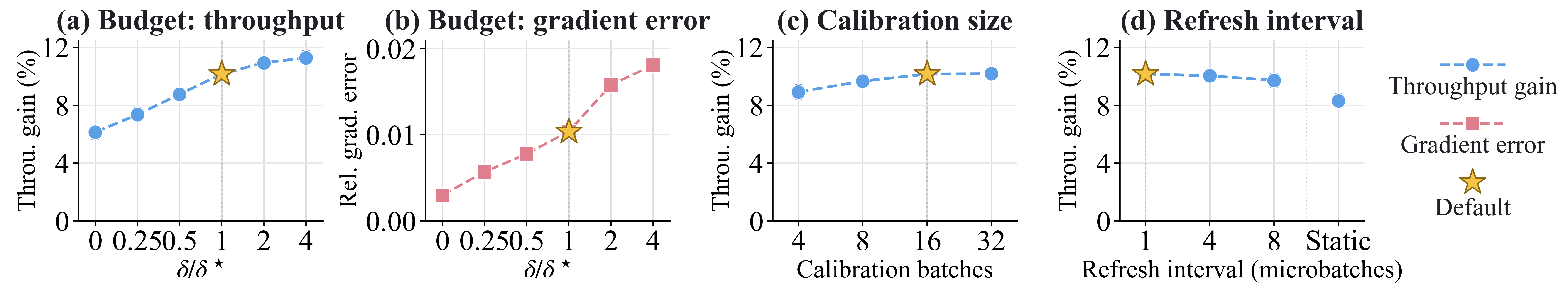}
    \vspace{-0.35in}
    \caption{Sensitivity to the internal fidelity budget $\delta$, calibration size, and risk-refresh interval.}
    \label{fig:hyperparameter-sensitivity}
    \vspace{-0.2in}
\end{figure*}

\vspace{-0.05in}
\section{Conclusion}

\method makes current-update sensitivity a criterion for learner-state recovery in GRPO. It combines policy-update exposure with measured recovery utility to allocate H/L/R actions under memory and calibrated risk budgets. With 2K responses and memory ceiling $1.10B_{\mathrm{GC}}$, actor-update throughput improves up to 13.47\% over GC and 7.87\% over fastest evaluated baseline, with downstream quality close to paired GC. This benefit spans mathematics, code, and logical reasoning across response lengths and architectures. These findings support the principle: recovery precision should follow the active update's sensitivity to saved states, alongside memory and recomputation costs. Future work will examine this principle in longer-context training and broader full-parameter settings.

\clearpage

\subsection*{AI Use Statement}
In this work, we used generative AI tools to assist with checking mathematical derivations, statements, and proofs, as well as for language polishing and refinement. We reviewed all AI-assisted work and verified the mathematical content. We take responsibility for the final content of this work, including the manuscript, mathematical claims and proofs, experimental results, and code.

\subsection*{Ethics Statement}
This work improves GRPO training efficiency using existing models and public datasets. Responsible application requires respecting model and dataset licenses and assessing the safety and potential misuse of downstream models in their intended deployment contexts.

\subsection*{Reproducibility Statement}
Sec.~\ref{sec:method} defines the GRPO objective, policy-update exposure, recovery utility, calibrated risk, and H/L/R allocation. Appendices~\ref{app:recovery-units} and~\ref{app:calibration-scheduling} detail graph attribution, profiling, calibration, scheduling, and validation; Appendix~\ref{app:theory} provides the theoretical assumptions and proofs. Appendix~\ref{app:experimental-setup} specifies data splits, execution settings, baseline implementations, configuration selection, fixed-workload timing, same-state gradient diagnostics, and on-policy training evaluation.

\bibliography{iclr2027_conference}
\bibliographystyle{iclr2027_conference}

\clearpage
\appendix

\begingroup
\hypersetup{linkcolor=black}
\begin{center}
{\Large\bfseries Appendix Contents}
\end{center}
\vspace{0.12in}
\fontsize{10}{16}\selectfont
\setlength{\parskip}{0pt}
\makeatletter
\renewcommand{\@dotsep}{6}
\makeatother

\appendixentry{app:recovery-units}{Recovery-Unit Construction and Profiling}
\appendixsubentry{app:graph-attribution}{Graph Attribution and Recomputation Verification}
\appendixsubentry{app:recovery-utility}{Measured Recovery Utility}
\appendixsubentry{app:recovery-interactions}{Non-Additive Recovery Interactions}

\appendixentry{app:calibration-scheduling}{Calibration and Scheduling Details}
\appendixsubentry{app:exposure-calibration}{Exposure and Low-Precision Calibration}
\appendixsubentry{app:allocation-validation}{Joint Allocation and Schedule Validation}
\appendixsubentry{app:online-scheduling}{Online Scheduling and Complexity}

\appendixentry{app:theory}{Theoretical Analysis}
\appendixsubentry{app:update-semantics}{Update Semantics and Exact Recovery}
\appendixsubentry{app:exposure-risk-theory}{Policy-Update Exposure and Calibrated Risk}

\appendixentry{app:additional-experiments}{Additional Experimental Results}
\appendixsubentry{app:experimental-setup}{Detailed Experimental Setup}
\appendixsubsubentry{app:setup-training-data}{Training Configuration and Data}
\appendixsubsubentry{app:setup-recovery-implementations}{Recovery-Method Implementations}
\appendixsubsubentry{app:setup-configuration-selection}{Configuration Selection}
\appendixsubsubentry{app:setup-efficiency-measurements}{Efficiency and Memory Measurements}
\appendixsubsubentry{app:setup-gradient-fidelity}{Gradient-Fidelity Evaluation}
\appendixsubsubentry{app:setup-downstream-evaluation}{Downstream Training and Evaluation}
\appendixsubentry{app:full-baseline-comparison}{Additional Baselines}
\appendixsubentry{app:exposure-details}{Policy-Update Exposure Details}
\appendixsubentry{app:frontier-details}{Canonical Memory--Throughput Frontier}
\appendixsubentry{app:generalization}{Generalization Across Tasks, Lengths, and Models}
\appendixsubsubentry{app:system-robustness}{Cross-Task Transfer}
\appendixsubsubentry{app:scaling-details}{Response-Length and Model Transfer}
\appendixsubsubentry{app:quality-full-parameter-transfer}{Training Quality and Full-Parameter Transfer}
\appendixsubsubentry{app:full-parameter-budget-planning}{Full-Parameter Memory--Throughput Trade-offs}
\appendixsubentry{app:shared-tolerance}{Sensitivity to the Shared Gradient-Error Tolerance}
\appendixsubentry{app:ablation-robustness}{Ablation, Sensitivity, and Runtime Details}

\appendixunnumberedentry{app:statements}{Limitations}
\endgroup

\vspace{0.15in}
\noindent Appendix~\ref{app:recovery-units} describes recovery-unit construction, recomputation verification, and utility profiling. Appendix~\ref{app:calibration-scheduling} specifies exposure calibration, risk estimation, and online H/L/R allocation with schedule validation. Appendix~\ref{app:theory} derives the update decomposition, establishes exact-recovery conditions, and analyzes isolated and joint gradient perturbations. Appendix~\ref{app:additional-experiments} presents experimental protocols, additional baselines, memory--throughput comparisons, generalization results, and sensitivity and runtime analyses.

\clearpage

\section{Recovery-Unit Construction and Profiling}
\label{app:recovery-units}

\subsection{Graph Attribution and Recomputation Verification}
\label{app:graph-attribution}

A recovery unit couples a forward subgraph with the complete state bundle required by its backward operators. \method traces graph lineage, module ownership, checkpoint boundaries, storage aliases, and backward dependencies, accounting for fused operations and shared checkpoint inputs. Let $\mathcal K(\mathbf a^R)$ and $\mathcal K(\mathbf a^{R\rightarrow H_u})$ denote recomputation-kernel multisets under all-\Ract recovery and with only unit $u$ materialized in high precision. Its removed footprint is $\mathcal R_u=\mathcal K(\mathbf a^R)\setminus\mathcal K(\mathbf a^{R\rightarrow H_u})$, with multiplicity preserved. A candidate is admitted only if $\mathcal R_u\neq\varnothing$ and the removed executions belong to its attributed forward subgraph; executions are matched by operator type, graph lineage, tensor shape, and profiler call stack. The attributed backward graph defines $\mathcal J(u)$ as response positions whose parameter-gradient contributions depend on unit $u$'s saved state, including indirect dependencies through later layers. The coefficient $\alpha_{t,u}$ supplies an empirical token weight within this support.

Each verified unit exposes feasible actions $\mathcal A_u\subseteq\{\Hact,\Lact,\Ract\}$: \Hact retains the complete bundle in training precision, \Lact reconstructs it from a supported low-precision component, and \Ract recomputes it from exact checkpoint inputs. Each action supplies the state required by the unit's backward operators.

\subsection{Measured Recovery Utility}
\label{app:recovery-utility}

Conditional utility compares recovery actions under the same candidate-storage plan $\mathbf q$:
\begin{equation}
\begin{aligned}
\Delta\tau_u(a\mid\mathbf q)
&=T_{\mathrm{upd}}(\mathbf q,\mathbf a^R)-T_{\mathrm{upd}}(\mathbf q,\mathbf a^{R\rightarrow a_u}),\\
\Delta m_u(a\mid\mathbf q)
&=\left[M_{\mathrm{peak}}(\mathbf q,\mathbf a^{R\rightarrow a_u})-M_{\mathrm{peak}}(\mathbf q,\mathbf a^R)\right]_+.
\end{aligned}
\label{eq:app-utility}
\end{equation}
Let $\mathbf q^R$ retain standard GC checkpoints. Candidate-capture overhead is
\begin{equation}
C(\mathbf q)=T_{\mathrm{upd}}(\mathbf q,\mathbf a^R)-T_{\mathrm{upd}}(\mathbf q^R,\mathbf a^R).
\label{eq:candidate-overhead}
\end{equation}
Net savings over GC therefore satisfy
\begin{equation}
\begin{aligned}
T_{\mathrm{upd}}(\mathbf q^R,\mathbf a^R)-T_{\mathrm{upd}}(\mathbf q,\mathbf a)=T_{\mathrm{upd}}(\mathbf q,\mathbf a^R)-T_{\mathrm{upd}}(\mathbf q,\mathbf a)-C(\mathbf q).
\end{aligned}
\label{eq:net-recovery-saving}
\end{equation}
Complete actor-update timing includes capture costs even for candidates subsequently discarded. We set $\Delta\tau_u(\Ract\mid\mathbf q)=\Delta m_u(\Ract\mid\mathbf q)=0$ and exclude non-\Ract actions with non-positive conditional utility. The positive part in Eq.~\ref{eq:app-utility} records additional peak-memory demand relative to same-plan all-\Ract execution. Profiling uses repeated, warmed-up paired updates with matched parameters, inputs, microbatch assignments, precision, RNG states, and execution settings. Peak-memory measurements start from matched fresh allocator states.

\subsection{Non-Additive Recovery Interactions}
\label{app:recovery-interactions}

Recovery utilities can interact through recomputation, fusion, and memory lifetimes. Overlapping footprints, $\mathcal R_u\cap\mathcal R_v\neq\varnothing$, identify dependent units. Bundles jointly required by an executable recovery dependency are merged during graph attribution; remaining interactions are assessed during schedule validation. For a selected non-\Ract action, its leave-one-out utility under schedule $\mathbf a$ is
\begin{equation}
\Delta\tau_u^{\mathrm{loo}}(\mathbf a\mid\mathbf q)=T_{\mathrm{upd}}(\mathbf q,\mathbf a_{u\rightarrow R})-T_{\mathrm{upd}}(\mathbf q,\mathbf a),
\label{eq:loo}
\end{equation}
where $\mathbf a_{u\rightarrow R}$ changes unit $u$ to \Ract while preserving capture under $\mathbf q$. During validation, actions with non-positive leave-one-out utility are removed and the schedule re-evaluated. This measures runtime contribution under joint execution, including overlapping recomputation and allocator effects.

\section{Calibration and Scheduling Details}
\label{app:calibration-scheduling}

\subsection{Exposure and Low-Precision Calibration}
\label{app:exposure-calibration}

\textbf{Exposure for the evaluated MLP units.} Let $\mathcal U_{L,b}$ denote the L-eligible units in microbatch $b$, and let $\mathcal T_b$ contain its valid generated response-token positions, excluding prompt and padding positions. Under response-validity masking, $m_t=1$ for every $t\in\mathcal T_b$. Each selected tensor spans the complete microbatch and has response-token support $\mathcal J(u)=\mathcal T_b$. Let $X_{u,b,t,:}\in\mathbb R^{d_u}$ denote its feature vector at the source position used to predict response token $t$. We compute tokenwise activation magnitudes from the original forward tensor:
\begin{equation}
a_{t,u,b}=\left(\frac{1}{d_u}\sum_{j=1}^{d_u}X_{u,b,t,j}^{2}\right)^{1/2},\qquad \alpha_{t,u,b}=\frac{a_{t,u,b}+\varepsilon_\alpha}{|\mathcal T_b|^{-1}\sum_{s\in\mathcal T_b}a_{s,u,b}+\varepsilon_\alpha}.
\label{eq:activation-token-weights}
\end{equation}
We use FP32 reductions and $\varepsilon_\alpha=10^{-8}$. The weights have mean one within each unit and microbatch, preserving relative token magnitudes while removing the unit's overall activation scale. They are detached from autograd and retained until the current update coefficients become available. Exposure is
\begin{equation}
\widehat{\mathcal E}_{u,b}=\sum_{t\in\mathcal T_b}|\omega_t|\alpha_{t,u,b}.
\label{eq:mlp-weighted-exposure}
\end{equation}
Masked positions have $\omega_t=0$. The weights provide an empirical activation-based proxy; calibration relates the resulting exposure and reconstruction distortion to measured single-unit gradient error. The ideal exposure in Eq.~\ref{eq:update-exposure} retains the state-to-gradient Jacobians.

\textbf{Low-precision recovery.}
The saved-tensor policy selects complete BF16 MLP tensors $X_u$ of at least $32$ MiB, with at most one selected tensor per layer. Each selected tensor spans the microbatch's token positions, and its complete token dimension participates in a single recovery action. The low-precision codec applies
\begin{equation}
c_L(X_u)=\operatorname{cast}_{\mathrm{E4M3FN}}(X_u),
\qquad
\widehat X_{L,u}=\operatorname{cast}_{\mathrm{BF16}}(c_L(X_u)).
\label{eq:fp8-roundtrip}
\end{equation}
The codec uses \texttt{torch.float8\_e4m3fn} without scales, zero-points, or outlier tensors. Forward uses the original BF16 tensor; backward uses its reconstruction, with other bundle components exact.

Calibration covers $\mathcal U_L=\{u\in\mathcal U:\Lact\in\mathcal A_u\}$. For each $u\in\mathcal U_L$ and microbatch $b\in\mathcal C$, paired passes match parameters, inputs, RNG states, microbatch assignments, and backward settings; only $u$ uses \Lact. Let $z_{u,b}$ and $\widehat z_{u,b}$ be the vectorized bundles before and after replacing $X_{u,b}$ with its reconstruction, and $g_{\mathrm{ex},b},g_{L,u,b}$ the corresponding trainable-parameter gradients. We measure
\begin{equation}
e^L_{u,b}=\frac{\|\widehat z_{u,b}-z_{u,b}\|_2}{\|z_{u,b}\|_2+\varepsilon},
\qquad
y_{u,b}=\sqrt{\frac{\|g_{L,u,b}-g_{\mathrm{ex},b}\|_2^2}{\max\{\|g_{\mathrm{ex},b}\|_2^2,\varepsilon_{\mathrm{num}}\}}}.
\label{eq:app-calibration-error}
\end{equation}
Calibration converts operands to FP32 before subtraction and norm computation, using $\varepsilon=10^{-8}$ and $\varepsilon_{\mathrm{num}}=10^{-30}$, distinct from clipping radius $\epsilon$. Diagnostics follow Appendix~\ref{app:setup-gradient-fidelity}.

Group eligible units by saved-state type and layer $c(u)$. With $x_{u,b}=\widehat{\mathcal E}_{u,b}e^L_{u,b}$, fit
\begin{equation}
\kappa_c=\operatorname*{arg\,min}_{\kappa\ge0}\sum\nolimits_{(u,b):\,c(u)=c}(y_{u,b}-\kappa x_{u,b})^2,
\label{eq:app-kappa}
\end{equation}
and define
\begin{equation}
\bar e^L_u=\frac{1}{|\mathcal C|}\sum\nolimits_{b\in\mathcal C}e^L_{u,b},
\qquad
\psi_u=\kappa_{c(u)}\widehat{\mathcal E}_u\bar e^L_u.
\label{eq:app-risk-score}
\end{equation}
Initialization fixes $\kappa_c$ and $\bar e_u^L$ within each unit/shape class. Each exposure refresh combines current token coefficients with the corresponding activation weights to update $\widehat{\mathcal E}_{u,b}$ and $\psi_{u,b}$ before recovery selection, as defined in Eq.~\ref{eq:token-weighted-risk}.

\subsection{Joint Allocation and Schedule Validation}
\label{app:allocation-validation}

Approximation costs are
\begin{equation}
\widehat d_u(\Hact)=\widehat d_u(\Ract)=0,\qquad
\widehat d_u(\Lact)=\psi_u,\qquad
\widehat D(\mathbf a)=\sum\nolimits_u\widehat d_u(a_u).
\label{eq:app-distortion}
\end{equation}
For microbatch $b$ and candidate plan $\mathbf q$, feasible assignments satisfy $a_u\in\mathcal A_u(q_u)$ for all $u$, $\sum_u\Delta m_u(a_u\mid\mathbf q)\le B-M_{\mathrm{peak}}(\mathbf q,\mathbf a^R)$, and $\widehat D(\mathbf a)\le\delta_b$. The allocator selects
\begin{equation}
\mathbf a^\star=\arg\max_{\mathbf a}\sum\nolimits_u\Delta\tau_u(a_u\mid\mathbf q).
\label{eq:app-joint-allocation}
\end{equation}
Unit-level quantities suppress the microbatch index.

\textbf{Conservative discretization.}
For $M_{\mathrm{res}}=B-M_{\mathrm{peak}}(\mathbf q,\mathbf a^R)>0$ and $\delta_b>0$, use $N_B=N_\delta=1024$ intervals:
\begin{equation}
h_B=\frac{M_{\mathrm{res}}}{N_B},\qquad h_\delta=\frac{\delta_b}{N_\delta},\qquad K_B=N_B+1,\qquad K_\delta=N_\delta+1.
\label{eq:allocation-grid}
\end{equation}
Including zero on each axis, round action costs upward:
\begin{equation}
b_u(a)=\left\lceil \Delta m_u(a\mid\mathbf q)/h_B \right\rceil,
\qquad
r_u(a)=\left\lceil \widehat d_u(a)/h_\delta \right\rceil.
\label{eq:allocation-integer-costs}
\end{equation}
Utilities remain unquantized. Each cost is bounded by its integer cost times the grid step, giving
\begin{equation}
\sum\nolimits_u b_u(a_u)\le N_B \Longrightarrow \sum\nolimits_u\Delta m_u(a_u\mid\mathbf q)\le M_{\mathrm{res}},\quad
\sum\nolimits_u r_u(a_u)\le N_\delta \Longrightarrow \widehat D(\mathbf a)\le\delta_b.
\label{eq:allocation-grid-feasibility}
\end{equation}
Discretization thus preserves additive-surrogate feasibility.

\textbf{Dynamic programming.}
Let $F_i(j,k)$ be the maximum utility over the first $i$ units with integer costs exactly $(j,k)$. Initialize $F_0(0,0)=0$ and other entries to $-\infty$. For $0\le j\le N_B$ and $0\le k\le N_\delta$,
\begin{equation}
F_i(j,k)=\max\nolimits_{a\in\mathcal A_i(q_i):\,b_i(a)\le j,\,r_i(a)\le k}
\left\{F_{i-1}(j-b_i(a),k-r_i(a))+\Delta\tau_i(a\mid\mathbf q)\right\}.
\label{eq:allocation-dp}
\end{equation}
Decomposing each assignment into a prefix and its final action proves the recurrence by induction. Maximizing over terminal states and backtracking solves the discretized surrogate. Terminal ties favor lower risk, then memory; within-state ties follow $\Ract,\Hact,\Lact$ in fixed unit order.

\textbf{Boundary cases.}
Zero residual memory permits only zero-memory-cost actions and one memory state. Zero risk disables \Lact and leaves one risk state. Plans exceeding $B$ under all-\Ract execution are excluded before capture; admitted plans retain the zero-cost all-\Ract assignment.

\textbf{Schedule validation.}
Calibration measures memory, runtime, and gradient deviation. Non-\Ract actions with non-positive leave-one-out utility are removed; schedules re-evaluated. Selection requires
\begin{equation}
\max\nolimits_{v\in\mathcal D_{\mathrm{dev}}} M_{\mathrm{peak}}(\mathbf q_v,\mathbf a_v;v)\le B,
\qquad
\frac{1}{|\mathcal D_{\mathrm{dev}}|}\sum\nolimits_{v\in\mathcal D_{\mathrm{dev}}}e_{\mathrm{grad},v}\le\epsilon_g.
\label{eq:app-measured-constraints}
\end{equation}
Each $v$ is a complete update with microbatch plans $\mathbf q_v,\mathbf a_v$. Error is measured after accumulation and distributed reduction, before clipping; $\epsilon_g=0.015$ constrains its development-set mean. Selection enforces Eq.~\ref{eq:configuration-selection}'s terminal validation criterion. Configurations remain fixed for held-out evaluation.

\subsection{Online Scheduling and Complexity}
\label{app:online-scheduling}

\textbf{Candidate-storage plans.}
Development evaluation selects the fastest feasible plan per supported shape class and memory budget using complete actor-update timing. Reservations cover retained candidates, exact checkpoints, encoding, and backward workspaces. Each microbatch selects its validated plan before forward; unsupported shapes use all-\Ract capture.

\textbf{Available recovery actions.}
Retained representations restrict verified actions:
\begin{equation}
\mathcal A_u(q_u)=\mathcal A_u\cap
\begin{cases}
\{\Hact,\Lact,\Ract\}, & q_u=\Hact,\\
\{\Lact,\Ract\}, & q_u=\Lact,\\
\{\Ract\}, & q_u=\Ract.
\end{cases}
\label{eq:available-recovery-actions}
\end{equation}
Post-forward encoding enters recovery cost. Compressed candidates assigned to \Ract are discarded and recomputed from exact checkpoints with execution states. Forward uses original activations.

\textbf{Current-update selection.}
Microbatch losses retain the complete update's normalization. During forward, \method collects the activation weights in Eq.~\ref{eq:activation-token-weights}. After forward, it computes the current token coefficients and forms each unit's exposure and risk score using Eq.~\ref{eq:token-weighted-risk}. Recovery allocation combines these tensor-level scores with the available actions, measured utilities, and memory costs. For optimized response-token counts $N_b$ with positive total, the microbatch risk budget is $\delta_b=\delta N_b/\sum_{b'}N_{b'}$. Allocation uses the plan's residual memory and $\delta_b$. These budgets constrain surrogate risk; Eq.~\ref{eq:measured-gradient-error} measures fidelity on the accumulated gradient.

\textbf{Implementation and complexity.}
Compiled single-threaded CPU code uses FP64 utility tables and risk costs. Actor-update timing includes input preparation, device-to-host transfers, synchronization, DP, and backtracking, alongside capture, exposure collection, and recovery.

Let $U$ count units, $U_L=|\mathcal U_L|$, $C$ calibration batches, $P$ profiling repetitions, $Q$ plans, and $A_{\max}\le3$ actions per unit. Profiling requires $O(QPUA_{\max})$ matched measurements; isolated calibration requires $O(CU_L)$ evaluations, plus complete executions for joint plan validation.

Allocation takes $O(UA_{\max}K_BK_\delta)=O(UK_BK_\delta)$ time. Two rolling value tables use $O(K_BK_\delta)$ storage; predecessor actions use $O(UK_BK_\delta)$ for backtracking. Collecting activation-based token weights costs $O(V)$ for $V$ reduced activation elements and retains $O(I)$ weight values. Given these weights, coefficient and score refresh takes $O(|\mathcal T|+I+U)$ per microbatch, with $O(U)$ additional unit-level metadata. Peak-memory accounting includes retained weights and reduction workspaces. Update timing includes collection, refresh, transfers, and allocation.

\section{Theoretical Analysis}
\label{app:theory}

\subsection{Update Semantics and Exact Recovery}
\label{app:update-semantics}

\begin{proof}[Proof of Proposition~\ref{prop:exact-decomposition}]
Fix rollout advantages, update masks, old-policy log-probabilities, and reference-policy quantities as stop-gradient values, with $N=\sum_t m_t>0$. Define $s_t=\log\pi_\theta(o_t\mid q,o_{<t})$ and $r_t=\exp(s_t-s_t^{\mathrm{old}})$. The token's policy loss is
\begin{equation}
\ell_t(s_t)=-\frac{m_t}{N}\min\!\left(r_t\widehat A_t,\operatorname{clip}(r_t,1-\epsilon,1+\epsilon)\widehat A_t\right).
\end{equation}
Since $\partial r_t/\partial s_t=r_t$, away from clipping boundaries,
\begin{equation}
\frac{\partial\ell_t}{\partial s_t}=-\frac{m_t}{N}\chi_t r_t\widehat A_t,
\qquad
\chi_t=\mathbf 1\!\left[(\widehat A_t\ge0\wedge r_t\le1+\epsilon)\vee(\widehat A_t<0\wedge r_t\ge1-\epsilon)\right].
\end{equation}
At boundaries, learner implementation determines the derivative. The token-separable KL term gives
\begin{equation}
\omega_t=\frac{\partial\mathcal L}{\partial s_t}=-\frac{m_t}{N}\chi_t r_t\widehat A_t+\beta\frac{\partial\mathcal L_{\mathrm{KL}}}{\partial s_t}.
\end{equation}
For the sampled-token K3 regularizer in experiments, with fixed reference log-probability $s_t^{\mathrm{ref}}$,
\begin{equation}
\mathcal L_{\mathrm{KL}}=\frac{1}{N}\sum_t m_t\left[\exp(s_t^{\mathrm{ref}}-s_t)-(s_t^{\mathrm{ref}}-s_t)-1\right],
\qquad
\frac{\partial\mathcal L_{\mathrm{KL}}}{\partial s_t}=\frac{m_t}{N}\left[1-\exp(s_t^{\mathrm{ref}}-s_t)\right].
\end{equation}
Because the objective depends on $\theta$ through $\{s_t\}$, the chain rule yields
\begin{equation}
\nabla_\theta\mathcal L
=\sum\nolimits_t\frac{\partial\mathcal L}{\partial s_t}\nabla_\theta s_t
=\sum\nolimits_t\omega_t g_t,
\qquad
g_t=\nabla_\theta s_t.
\end{equation}
Coefficients come from rollout data and the exact forward pass in linear time. Recovery modifies saved states, preserving forward log-probabilities and loss-side quantities, so $\{\omega_t\}$ remains fixed.
\end{proof}

Checkpointing reconstructs saved states by replaying forward computation~\citep{chen2016training}. Exact H/R equivalence follows under matched execution.

\begin{proposition}[Exact recovery equivalence]
\label{prop:exact-recovery}
Fix the forward pass and backward graph. Assume \Hact preserves complete saved states and \Ract reproduces them with identical parameters, inputs, positional/RNG states, and deterministic operator and precision settings. Under identical deterministic backward execution and accumulation/reduction orders, every H/R-only schedule $\mathbf a$ satisfies
\begin{equation}
\mathcal L^{\mathbf a}=\mathcal L^{\mathbf a^R},
\qquad
g^{\mathbf a}=g^{\mathbf a^R},
\end{equation}
where $g^{\mathbf a}$ denotes the computed parameter gradient.
\end{proposition}

\begin{proof}
Forward execution and states coincide. From equal loss derivatives, reverse-topological induction gives equal backward outputs; matched accumulation/reduction preserve gradients.
\end{proof}

This equivalence motivates zero H/R approximation costs.

\subsection{Policy-Update Exposure and Calibrated Risk}
\label{app:exposure-risk-theory}

Fix the exact forward pass, its loss-side quantities, and $\{\omega_t\}$. The following perturbation analysis considers the backward maps in real arithmetic with a fixed execution graph. For unit $u$, let $z_u$ be its exact saved state and $\widehat z_u=z_u+\Delta z_u$ its reconstruction. Define $g_t(z)$ as the parameter-gradient vector returned by seeding a unit adjoint at $s_t$ when unit $u$ consumes $z$ and all other saved states remain exact. Thus $g_t(z_u)=\nabla_\theta s_t$. For fixed saved states, backward propagation is linear in these seed adjoints. Tokens outside $\mathcal J(u)$ are unaffected, giving the exact identity
\begin{equation}
\Delta g_u=\sum\nolimits_{t\in\mathcal J(u)}\omega_t\left[g_t(z_u+\Delta z_u)-g_t(z_u)\right].
\label{eq:exact-gradient-perturbation}
\end{equation}

\begin{theorem}
\label{thm:exposure-bound}
For each $t\in\mathcal J(u)$, assume $g_t$ is continuously differentiable on a neighborhood containing the segment $[z_u,\widehat z_u]$, with an $H_{t,u}$-Lipschitz Jacobian in operator norm. Define
\begin{equation}
G_{t,u}=\left.\frac{\partial g_t(z)}{\partial z}\right|_{z=z_u}.
\label{eq:state-gradient-jacobian}
\end{equation}
Then
\begin{equation}
\Delta g_u=\sum\nolimits_{t\in\mathcal J(u)}\omega_tG_{t,u}\Delta z_u+r_u,
\label{eq:first-order-gradient-perturbation}
\end{equation}
where
\begin{equation}
\|r_u\|_2\le\frac12\mathcal H_u\|\Delta z_u\|_2^2,
\qquad
\mathcal H_u=\sum\nolimits_{t\in\mathcal J(u)}|\omega_t|H_{t,u}.
\label{eq:second-order-remainder}
\end{equation}
Consequently,
\begin{equation}
\|\Delta g_u\|_2\le\mathcal E_u\|\Delta z_u\|_2+\frac12\mathcal H_u\|\Delta z_u\|_2^2,
\label{eq:exposure-bound}
\end{equation}
with
\begin{equation}
\mathcal E_u=\sum\nolimits_{t\in\mathcal J(u)}|\omega_t|\|G_{t,u}\|_{\mathrm{op}}.
\label{eq:ideal-exposure}
\end{equation}
\end{theorem}

\begin{proof}
The integral Taylor formula gives
\begin{equation}
\begin{aligned}
g_t(z_u+\Delta z_u)&=g_t(z_u)+G_{t,u}\Delta z_u+r_{t,u},\\
r_{t,u}&=\int_0^1\left[Dg_t(z_u+\lambda\Delta z_u)-Dg_t(z_u)\right]\Delta z_u\,d\lambda.
\end{aligned}
\label{eq:token-taylor}
\end{equation}
Lipschitz continuity implies
\begin{equation}
\|r_{t,u}\|_2\le\int_0^1\lambda H_{t,u}\|\Delta z_u\|_2^2\,d\lambda
=\frac12H_{t,u}\|\Delta z_u\|_2^2.
\label{eq:token-remainder}
\end{equation}
Substitution into Eq.~\ref{eq:exact-gradient-perturbation} gives Eq.~\ref{eq:first-order-gradient-perturbation} with $r_u=\sum_t\omega_t r_{t,u}$. The triangle inequality bounds its norm by $\mathcal H_u\|\Delta z_u\|_2^2/2$. Applying $\|G_{t,u}\Delta z_u\|_2\le\|G_{t,u}\|_{\mathrm{op}}\|\Delta z_u\|_2$ yields Eq.~\ref{eq:exposure-bound}.
\end{proof}

The bound separates reconstruction magnitude from current-update sensitivity. For fixed perturbation norm, $\mathcal E_u$ controls the first-order bound; realized deviation depends on direction and cancellation.

\begin{corollary}[Zero-coefficient support]
\label{cor:zero-exposure}
If $\omega_t=0$ for every $t\in\mathcal J(u)$, then any state perturbation for which the fixed backward map remains well-defined satisfies
\begin{equation}
\Delta g_u=0.
\label{eq:zero-exposure-exact}
\end{equation}
Under the theorem's smoothness assumptions, $\mathcal E_u=\mathcal H_u=0$.
\end{corollary}

\begin{proof}
Every term in Eq.~\ref{eq:exact-gradient-perturbation} has zero coefficient. This identity requires no Taylor approximation.
\end{proof}

When $\beta=0$, masked positions, zero advantages, and flat clipped policy-loss branches yield zero coefficients. The corollary applies when this holds throughout the unit support. With KL regularization, unmasked positions with zero policy-loss coefficients retain $\omega_t=\beta\,\partial\mathcal L_{\mathrm{KL}}/\partial s_t$; masked positions have zero coefficient under the masked K3 objective.

\textbf{From isolated perturbations to joint risk.}
For one microbatch, let $Z=(z_1,\ldots,z_U)$ collect unit-level recovery-state coordinates and $\mathcal G(Z')$ denote the gradient returned with states $Z'$, holding the forward coefficients fixed. Shared aliases must be represented consistently so that these coordinates describe jointly realizable interventions; unchanged shared checkpoints remain fixed. Equip the concatenated coordinates with the product Euclidean norm. For $\mathcal S=\{u:a_u=\Lact\}$, set $\Delta z_u=0$ outside $\mathcal S$ under Proposition~\ref{prop:exact-recovery}, and write $\Delta Z=(\Delta z_1,\ldots,\Delta z_U)$. Differentiation concerns reconstructed real-valued states, not the codec.

\begin{lemma}
\label{lem:joint-risk}
Assume $\mathcal G$ is continuously differentiable with an $H_{\mathrm{joint}}$-Lipschitz Jacobian on a convex neighborhood containing the joint and isolated perturbation segments. Then
\begin{equation}
\begin{gathered}
\Delta g_{\mathrm{joint}}:=\mathcal G(Z+\Delta Z)-\mathcal G(Z)
=\sum\nolimits_{u\in\mathcal S}v_u+r_{\mathrm{joint}},
\qquad
v_u:=\sum\nolimits_{t\in\mathcal J(u)}\omega_tG_{t,u}\Delta z_u,\\
\|r_{\mathrm{joint}}\|_2
\le\frac{H_{\mathrm{joint}}}{2}
\sum\nolimits_{u\in\mathcal S}\|\Delta z_u\|_2^2.
\end{gathered}
\label{eq:joint-gradient-expansion}
\end{equation}
Therefore,
\begin{equation}
\|\Delta g_{\mathrm{joint}}\|_2\le\sum\nolimits_{u\in\mathcal S}\mathcal E_u\|\Delta z_u\|_2+\frac{H_{\mathrm{joint}}}{2}\sum\nolimits_{u\in\mathcal S}\|\Delta z_u\|_2^2.
\label{eq:joint-exposure-bound}
\end{equation}
Let $g_{\mathrm{ex}}=\mathcal G(Z)$ and let $\Delta Z^{(u)}$ contain only block $\Delta z_u$. Define
\begin{equation}
c=\sqrt{\max\{\|g_{\mathrm{ex}}\|_2^2,\varepsilon_{\mathrm{num}}\}},
\qquad
y_u=\frac{\|\mathcal G(Z+\Delta Z^{(u)})-\mathcal G(Z)\|_2}{c}.
\label{eq:joint-isolated-risk}
\end{equation}
Then
\begin{equation}
\frac{\|\Delta g_{\mathrm{joint}}\|_2}{c}\le\sum\nolimits_{u\in\mathcal S}y_u+\frac{H_{\mathrm{joint}}}{c}\sum\nolimits_{u\in\mathcal S}\|\Delta z_u\|_2^2.
\label{eq:joint-relative-risk}
\end{equation}
\end{lemma}

\begin{proof}
At the exact state, $D_u\mathcal G(Z)=\sum\nolimits_{t\in\mathcal J(u)}\omega_tG_{t,u}$. Taylor expansion and $\|\Delta Z\|_2^2=\sum\nolimits_u\|\Delta z_u\|_2^2$ give Eq.~\ref{eq:joint-gradient-expansion}; the triangle inequality yields Eq.~\ref{eq:joint-exposure-bound}.

Each isolated deviation has the form $\mathcal G(Z+\Delta Z^{(u)})-\mathcal G(Z)=v_u+\rho_u$, with $\|\rho_u\|_2\le H_{\mathrm{joint}}\|\Delta z_u\|_2^2/2$. Hence
\begin{equation}
\Delta g_{\mathrm{joint}}=\sum\nolimits_{u\in\mathcal S}\left[\mathcal G(Z+\Delta Z^{(u)})-\mathcal G(Z)\right]+r_{\mathrm{joint}}-\sum\nolimits_{u\in\mathcal S}\rho_u.
\end{equation}
The remainder has norm at most $H_{\mathrm{joint}}\sum\nolimits_u\|\Delta z_u\|_2^2$. Applying the triangle inequality and dividing by $c$ proves Eq.~\ref{eq:joint-relative-risk}.
\end{proof}

\textbf{Connection to calibrated risk.}
For an isolated intervention, let $s_u=\|z_u\|_2+\varepsilon$, so $\|\Delta z_u\|_2=s_u e_u^L$. Theorem~\ref{thm:exposure-bound} gives
\begin{equation}
y_u\le\frac{s_u}{c}\mathcal E_u e_u^L+\frac{\mathcal H_u s_u^2}{2c}(e_u^L)^2.
\label{eq:relative-exposure-bound}
\end{equation}
\method estimates isolated relative error by $\psi_u=\kappa_{c(u)}\widehat{\mathcal E}_u\bar e_u^L$, using exact $\omega_t$, structural support, and proxy $\alpha_{t,u}$; $\kappa_{c(u)}$ absorbs residual sensitivity and state/gradient norm scaling. Lemma~\ref{lem:joint-risk}'s first-order bound and second-order interactions motivate the calibrated surrogate $\widehat D(\mathbf a)=\sum_{u\in\mathcal S}\psi_u$. Eq.~\ref{eq:app-measured-constraints} checks fidelity after joint recovery and gradient accumulation.

\section{Additional Experimental Results}
\label{app:additional-experiments}

\subsection{Detailed Experimental Setup}
\label{app:experimental-setup}

\subsubsection{Training Configuration and Data}
\label{app:setup-training-data}

\textbf{Training configuration.}
Unless otherwise specified, LoRA uses rank $r=16$, scaling $\alpha/r$ with $\alpha=32$, and zero dropout; only adapters are trainable. Adapter targets are \texttt{q/k/v/o} and \texttt{gate/up/down} for Qwen/Llama, and \texttt{qkv/o} and \texttt{gate\_up/down} for Phi. We use PyTorch 2.5.1/CUDA 12.4 with bundled NCCL, Transformers 4.46.3, eager execution, FlashAttention-backed PyTorch SDPA, and method-specific kernels. Default prompt/response limits are 1024/2048 tokens; rollout temperature is 0.6, with top-$p$ 1 and no top-$k$ filtering. Microbatch targets count prompt/response tokens, including padding only in dense execution; over-target sequences form singleton microbatches within the length limits. Within each workload, recovery methods share training and batching settings (Table~\ref{tab:execution-configurations}), except for specified packing/length variations.

\begin{table}[t]
\centering
\caption{Shared learner execution settings; LoRA defaults apply unless specified.}
\label{tab:execution-configurations}
\footnotesize
\begin{tabularx}{\linewidth}{@{}p{0.12\linewidth}X@{}}
\toprule
Configuration & Setting \\
\midrule
Parallelism & Four-rank FSDP \texttt{FULL\_SHARD}; decoder-layer wrapping; TP=PP=SP=1; checkpoint regions inside wrappers \\
FSDP options & \texttt{use\_orig\_params=True} for LoRA, \texttt{False} for full-parameter training; no forward prefetch; \texttt{BACKWARD\_PRE}; all-gather limiting \\
Precision & FP32 persistent parameters, trainable-parameter AdamW moments, gradient reduction/accumulation, and floating-point buffers; BF16 parameter computation \\
Updates & Gradient reduction after each microbatch backward; gradients accumulate over the complete rollout batch before a single optimizer step \\
Checkpoint & Non-reentrant checkpointing with RNG preservation and early stopping; TF32 and strict deterministic mode disabled; separate model/method RNG streams \\
Residency & GPU actor; CPU-offloaded reference; rollout weights and KV caches released before actor updates \\
LoRA & 512 prompts $\times$ 8 responses/update; 4096-token microbatch target/GPU; no padding removal \\
DSR-sub & 128 prompts $\times$ 8 responses/update; 24000-token microbatch target/GPU; padding removal; 1024/3072 prompt/response limits \\
\bottomrule
\end{tabularx}
\vspace{-0.2in}
\end{table}

\textbf{GRPO and optimization.}
Binary rewards indicate correct math answers, complete K\&K identity assignments, or passing all training code tests (10-second timeout per test). Group advantages use population-standard-deviation normalization with a $10^{-6}$ denominator stabilizer; equal-reward groups have zero advantage. Main training uses one actor epoch per rollout batch, a clipping radius of 0.2, and a K3 KL coefficient of $10^{-3}$. Actor parameters remain fixed while gradients accumulate over all microbatches in the batch, followed by one optimizer step. Rollout sampling and old/current policy likelihoods use the same temperature-scaled distribution with $T=0.6$. The old policy is the parameter snapshot used to generate the batch. During each training update, $r_t=1$ in exact arithmetic throughout gradient accumulation, and policy-ratio clipping is inactive. Policy and KL losses are normalized by all valid response tokens in the complete rollout batch.
The update mask includes generated response tokens through the first termination token, when present, and excludes prompt and padding positions. Length-truncated responses retain their generated tokens. Zero-advantage responses remain included in the KL objective.
AdamW uses betas $(0.9,0.999)$, epsilon $10^{-8}$, zero weight decay, and global gradient-norm clipping at 1.0. Learning rates are selected by terminal validation performance after 60 GC updates on DeepMath10K with 2K responses and seed 2026, from $\{2\times10^{-6},5\times10^{-6},10^{-5}\}$ plus $10^{-6}$ for Llama-3.1-8B. Selected rates are $2\times10^{-6}$ for Qwen2.5-3B and $5\times10^{-6}$ for Phi-3.5-mini/Llama-3.1-8B, transferred across tasks and lengths. Full-parameter Llama-3.2-3B-Instruct on DSR-sub uses a prespecified $10^{-6}$. Rates remain constant and are shared across recovery methods within each workload; other optimizer/GRPO hyperparameters are shared across models. Search costs count as shared preparation.

\textbf{Data construction.}
Table~\ref{tab:data-splits} gives splits after deduplication, grouping of identified variants, and benchmark-overlap removal using source identifiers or normalized text. Within each workload, split seed 2026 fixes problem lists across methods, training seeds, models, and lengths. Training data supply on-policy rollouts and calibration that preserves model/optimizer states. Validation supports configuration selection; validation/test data are excluded from parameter updates and calibration.

\begin{table*}[t]
\centering
\caption{Training and validation splits.}
\label{tab:data-splits}
\footnotesize
\begin{tabularx}{0.9\linewidth}{@{}>{\raggedright\arraybackslash}p{0.15\linewidth}>{\raggedleft\arraybackslash}p{0.07\linewidth}>{\raggedleft\arraybackslash}p{0.09\linewidth}>{\raggedright\arraybackslash}X@{}}
\toprule
Workload & Train & Validation & Construction \\
\midrule
DeepMath10K\newline \citep{he2025deepmath} & 10,000 & 2,048 & Disjoint DeepMath-103K subsets; topic/difficulty stratification \\
TACO-Verified\newline \citep{li2023taco} & 11,874 & 1,024 & Source/difficulty stratification; remaining eligible problems form training set \\
K\&K\newline \citep{xie2025logicrl} & 4,000 & 1,000 & Published three- to seven-person training partitions; 800/200 train/validation split per 1,000-problem partition \\
DSR-sub & 1,209 & --- & Existing training subset; frozen recovery-hyperparameter transfer \\
\bottomrule
\end{tabularx}
\vspace{-0.2in}
\end{table*}

\subsubsection{Recovery-Method Implementations}
\label{app:setup-recovery-implementations}

\textbf{Baselines.}
Each method retains its supported tensor coverage; \method's eligibility restrictions apply only to \method. In LoRA runs, sensitivity estimates use trainable-adapter gradients, and low-rank backward methods preserve input-gradient propagation through frozen layers. Unsupported operators use ordinary autograd. Quantization metadata, projection state, and workspaces count toward memory; GACT, ALAM, and Adacc use FP32 quantization metadata. Compatibility mappings and shape-dependent handling are recorded with the implementations.
\begin{itemize}[leftmargin=1.3em,itemsep=2pt,topsep=3pt]
\item \textbf{GC / No-GC}~\citep{chen2016training}: decoder-layer checkpointing and ordinary autograd retention define All-R and All-H.
\item \textbf{SAC}: verified selective-checkpoint policies retain GEMM outputs, fused-attention outputs, or both; uncovered layers use ordinary GC.
\item \textbf{Rockmate}~\citep{zhao2023rockmate}: native rematerialization optimization uses profiled operation costs and decoder-layer or attention/MLP partitions respecting FSDP boundaries.
\item \textbf{GACT}~\citep{liu2022gact}: sensitivity-based precision allocation uses quantization groups of 256 elements and refreshes sensitivity every 1000 updates.
\item \textbf{ALAM}~\citep{woo2024alam}: native mixed-precision allocation uses Average Quantization and GradNormVar. The 0.5-bit option averages four adjacent values before two-bit quantization, independently of the quantization group size. GradNormVar compares layerwise gradient-norm vectors while changing only the target activation's quantization seed. Calibration occurs initially and at 10\% of the training horizon.
\item \textbf{CompAct}~\citep{shamshoum-etal-2025-compact}: Sketches $P_{ij}\sim\mathcal N(0,1/r)$ store $XP$ and reconstruct $XPP^\top$ for backward, retaining full-size AdamW states. Seeds refresh every 200 updates.
\item \textbf{PRAC}~\citep{li2026prac}: principal and random-complement projections retain layer-dependent ranks, shared-input reuse, and complement scaling $(d-r_1)/r_2$, where $d$ is feature dimension and $r_1,r_2$ are subspace ranks. Both subspaces refresh every 200 updates.
\item \textbf{INSTANT}~\citep{Doan2026INSTANT}: ordinary forward computation uses separately calibrated token-axis projections for saved activations and output gradients, preserving native low-rank weight- and input-gradient computations. Calibration uses five microbatches initially and every 50 updates.
\item \textbf{AGoQ}~\citep{lin2026agoq}: eligible non-attention activations use FP4 with block size 128; attention activations retain execution precision. Native recomputation and supported fused kernels remain; parameter-gradient storage and communication follow the shared FSDP configuration.
\item \textbf{Adacc}~\citep{chen2025adacc}: tensor-dependent INT4 compression retains original-precision outliers, native quantization layouts, and adaptive compression/recomputation scheduling. Monitoring intervals double from 1 to a maximum of 64 updates.
\end{itemize}

Rockmate optimizes blockwise schedules from profiled costs~\citep{zhao2023rockmate}; throughput also reflects execution and distributed-runtime interactions. FSDP-compatible partitions are evaluated with full actor-update timing, including recomputation, backward, communication, and optimizer work, yielding lower throughput than decoder-layer GC at the reported operating points.

\textbf{HiLoRe.}
Sixteen microbatches calibrate $\kappa_c$ and $\bar e_u^L$ using activation-weighted exposure. Tokenwise activation weights are collected during forward, and exposure and risk refresh after forward. L applies an unscaled BF16--FP8 E4M3FN--BF16 round trip to complete MLP tensors $\geq32$ MiB, with at most one tensor per layer and one recovery action shared across its token dimension. Mixed capture reserves $p_H$ for H and $1-p_H$ for L, ranking bundles by time saved per byte and breaking ties by execution order; All-R is available. Memory accounting includes checkpoints and encoding/backward workspaces, validated jointly. DP uses 1024 intervals (1025 states) per positive-budget axis, $h_B=M_{\mathrm{res}}/1024$, $h_\delta=\delta_b/1024$, upward-rounded costs, and single-threaded CPU FP64 values with backtracking; zero risk disables L. Sensitivity sweeps vary $\delta^\star$, calibration size, or refresh interval individually, using nested subsets from 32 microbatches.

\subsubsection{Configuration Selection}
\label{app:setup-configuration-selection}

\textbf{Search space.}
Primary tuning uses Qwen2.5-3B/DeepMath10K with 2K responses and $B=1.10B_{\mathrm{GC}}$. For each workload/protocol, $B_{\mathrm{GC}}$ is maximum GC memory across ranks and measured updates under selected configuration. Each tunable method receives at most 16 configurations, including native references and refinements; GC/No-GC are fixed in the fixed-packing comparisons. Table~\ref{tab:method-configurations} summarizes tuning spaces, including checkpointing on/off for compression-only baselines.

\begin{table*}[t]
\centering
\caption{Method-specific tuning spaces and configuration counts. CKPT denotes decoder-layer checkpointing on/off.}
\label{tab:method-configurations}
\scriptsize
\begin{tabularx}{\linewidth}{@{}p{0.1\linewidth}Xr@{}}
\toprule
Method & Search space & Configs \\
\midrule
GC / No-GC & Fixed All-R / All-H & 2 fixed \\
SAC & Retain GEMM, attention, or both; decoder-layer or attention/MLP regions; all or even layers & 12 \\
Rockmate & Decoder-layer or attention/MLP partitions; activation-memory allowance fractions $\{0.90,0.95,1.00\}$; DP memory units $\{1,4\}$ MiB & 12 \\
GACT & Average bits $\{2,3,4,8\}$; group size 256; CKPT & 8 \\
ALAM & Average bits $\{0.5,2,8\}$; group size 256; CKPT; one group-size-128 refinement per precision & 9 \\
CompAct & Rank ratio $\{1/4,1/2,7/8\}$; CKPT & 6 \\
PRAC & Multipliers $\{1/2,1,3/2\}$ on linear/nonlinear rank ratios $(0.6,0.4)$; principal fraction $\{1/2,3/4\}$; CKPT & 12 \\
INSTANT & Energy thresholds $\{0.95,0.99,0.999\}$; oversampling 7; CKPT; three oversampling-16 comparisons and two threshold refinements & 11 \\
AGoQ & Two activation policies $\times$ five outer checkpoint placements, defined below; block size 128 & 10 \\
Adacc & INT4; outlier Z-score $\{2.0,2.5,3.0,3.5\}$; activation-memory allowance fractions $\{0.90,0.95,1.00\}$; native layouts and adaptive monitoring & 12 \\
\method & All-R capture, or $p_H\in\{0,1/2,1\}$ paired with $\delta\in\{0,0.010,0.020\}$ & 10 \\
\bottomrule
\end{tabularx}
\vspace{-0.2in}
\end{table*}

CompAct/PRAC rank ratios refer to feature dimensions, independently of LoRA rank, with both PRAC subspaces nonempty. Activation-memory allowances preserve the common device-memory ceiling $B$. MILP solves use two CPU threads, a 300-second limit, and a 1\% relative gap, recording incumbents and termination status. AGoQ compares its native activation policy and original-precision RMSNorm retention under five outer checkpoint placements: none, all decoder layers, attention-only, MLP-only, or even-indexed decoder layers. These preserve internal recomputation and FSDP boundaries; compatible placements are also available when refining other compression baselines.

\textbf{Pilots and refinements.}
Each candidate runs 10 updates with seed 2026, measuring throughput, peak memory, validation score, and same-state gradient errors at updates 1, 5, 10; horizon-dependent schedules use the 60-update horizon. Refinements inherit settings from finite, in-budget pilots with mean gradient error $\le0.015$, selected by the throughput and tie-breaking rule below; if none passes, the lowest-error finite, in-budget pilot is used. Compatibility checks cover operators, shapes, and checkpoints. Equivalent Rockmate schedules are evaluated once, with schedule-construction costs recorded separately; INSTANT deduplication requires equivalent projection and refresh behavior over the workload, not merely matching ranks. Configuration-induced OOM/numerical failures count toward the 16-configuration limit but yield no performance result; retries and repeated measurements add no candidates. Search ends at the limit or when refinements are exhausted.

ALAM's group-size-128 refinements and INSTANT's oversampling-16 comparisons inherit the preferred checkpoint setting at each precision or energy threshold. INSTANT thresholds 0.995 and 0.9995 inherit checkpoint and oversampling settings from the preferred 0.99 and 0.999 configurations, respectively. Further refinements increase precision/rank for fidelity or compression/reduce retention for memory; upward rank-ratio and energy-threshold refinements move halfway toward one. Up to three finite, in-budget pilots become finalists: the native reference, the fastest with mean gradient error $\le0.015$, and the highest-validation-score candidate. The lowest-error candidate fills the second slot if necessary; duplicates or unavailable choices are replaced in validation-score order.

\textbf{Final selection.}
Finalists complete 60 updates with seeds 2026 and 2027, continuing the seed-2026 pilot with model, optimizer, RNG, and method states preserved. GC uses paired initialization and prompt schedules. For completed finalists $\mathcal C$, the feasible set is
\begin{equation}
\mathcal F(\epsilon_g)=\left\{a\in\mathcal C:\ M_a^{\max}\le B,\quad \bar e_{\mathrm{grad},a}\le\epsilon_g,\quad Q_a\ge Q_{\mathrm{GC}}-0.01\right\}.
\label{eq:configuration-selection}
\end{equation}
Here, $M_a^{\max}$ covers both runs and all ranks, $\bar e_{\mathrm{grad},a}$ averages prescribed same-state gradient checks, and $Q_a$ averages terminal validation scores across seeds. The default $\epsilon_g=0.015$ permits 1.5\% mean relative gradient error; scores are fractions, so the quality tolerance is one absolute percentage point. All constraints are fixed before tuning and shared across methods. For nonempty $\mathcal F(\epsilon_g)$, define $v_{\max}=\max_{a\in\mathcal F(\epsilon_g)}\bar v_a$, where $\bar v_a$ is mean actor-update throughput across tuning seeds, including online method overhead. Among candidates with $\bar v_a\ge0.995v_{\max}$, selection favors validation score, then lower gradient error and memory. If $\mathcal F$ is empty, no feasible configuration is reported; any separately shown highest-quality finalist is labeled with its violated constraints. The internal risk budget $\delta$ is selected under these criteria. Table~\ref{tab:delta-sensitivity} reports sensitivity to $\delta$, and Appendix~\ref{app:shared-tolerance} examines sensitivity to the shared gradient-error tolerance $\epsilon_g$.

\textbf{Transfer.}
Additional models/tasks/lengths evaluate up to four deduplicated candidates fixed before transfer: the primary selection, nearest higher-fidelity and lower-memory alternatives, and native reference. Selection follows the same procedure on the corresponding task-validation set. DSR-sub transfers frozen recovery hyperparameters with renewed profiling/calibration and no task-specific search. Final configurations are frozen before testing, and test scores never enter selection.

\subsubsection{Efficiency and Memory Measurements}
\label{app:setup-efficiency-measurements}

\textbf{Replay measurements.}
To isolate recovery-allocation efficiency under a fixed workload, main LoRA efficiency results use matched replays with identical parameters, rollout data, update tokens, and microbatch assignments across methods. Five paired repetitions follow warm-up. Gain is $100(\bar v_a/\bar v_{\mathrm{GC}}-1)$, where $\bar v$ is mean throughput across repetitions; variability is the sample SD of repetition-level paired gains, in percentage points. Budget, packing, ablation, sensitivity, and runtime-decomposition experiments use eight fixed replay records disjoint from configuration-selection records, with five warmed-up paired repetitions. Each repetition restores model, optimizer, method, and RNG states, including refresh age. Budget sweeps use $B/B_{\mathrm{GC}}\in\{1.00,1.05,1.10,1.15,1.20\}$ under common memory/error criteria; budget-aware methods regenerate schedules, while unchanged configurations reuse measurements. Timings include refresh events executed during replay.

\textbf{Timing and throughput.}
Synchronized actor-update time includes forward, CPU scheduling, capture/compression, recovery, backward, communication, optimizer work, gradient-snapshot overhead, and required online adaptation or validation. The maximum elapsed time across ranks defines the actor-update time. Throughput counts non-padding prompt and response tokens across ranks once per input sequence, excluding recomputation, and divides this count by the measured actor-update time. The timed region covers the actor update; rollout generation, log-probability evaluation outside the actor update, benchmark evaluation, and reporting-only gradient diagnostics lie outside this region.

\textbf{Memory.}
NVML samples memory every 10 ms. Reported actor-update peaks and budget checks use maxima across ranks, measured updates, and repetitions. PyTorch allocated/reserved peaks supplement transient checks near budget boundaries. These measurements cover candidate capture and backward recovery, including method-specific retained states and workspaces.

\textbf{Setup and online overhead.}
One-time setup, outside the timed actor-update region, includes graph processing, utility profiling, initial calibration, candidate-plan construction and validation, and required reference-gradient computations. Initialization-calibrated statistics are reused during training. Reported speedups characterize steady-state actor-update efficiency, including online exposure refresh and allocation.

\subsubsection{Gradient-Fidelity Evaluation}
\label{app:setup-gradient-fidelity}

\textbf{Training-update fidelity.}
Checks occur at updates $\{1,5,10,20,30,40,50,60\}$ for LoRA and update 1 plus every 20 updates for DSR-sub. Replays for the main training evaluations compare each method with GC using identical saved pre-update parameters, rollout/loss inputs, microbatch assignments, and model RNG states. Method-specific calibration/projection states, schedules, refresh counters, and RNG states are also restored; checkpoint capture is timed separately from learner updates. Post-update mechanism analyses follow the separate protocol in Appendix~\ref{app:exposure-details}. All trainable-parameter gradients are compared after accumulation/reduction and before gradient-norm clipping:
\begin{equation}
e_{\mathrm{grad},a,s,t}=
\sqrt{\frac{\sum_j\lVert g_{a,s,t,j}-g_{\mathrm{GC}\mid a,s,t,j}\rVert_2^2}
{\max\{\sum_j\lVert g_{\mathrm{GC}\mid a,s,t,j}\rVert_2^2,10^{-30}\}}},
\label{eq:measured-gradient-error}
\end{equation}
where disjoint matching shards $j$ count each parameter once. FP64 statistics yield within-run means, across-seed summaries, and maximum error.

\textbf{Numerical controls.}
Independent GC executions and H/R-only recovery are compared against GC on the same saved Qwen2.5-3B / DeepMath10K updates, fixing parameters, inputs, loss weights, RNG states, and microbatch assignments. Execution-matched controls additionally fix deterministic operators, precision settings, and gradient accumulation/reduction orders, following Proposition~\ref{prop:exact-recovery}. Errors use Eq.~\ref{eq:measured-gradient-error} after accumulation/reduction and before clipping.

\subsubsection{Downstream Training and Evaluation}
\label{app:setup-downstream-evaluation}

\textbf{Training and reporting.}
Final LoRA quality comparisons use 60 updates with seeds 1234, 2025, and 3407; DSR-sub uses 200 updates with seed 1234. Within each seed, methods share initialization and prompt schedules but generate their own on-policy responses. Configurations remain frozen, and observed constraint violations are reported. Final test percentages are summarized as across-seed means, sample SDs, and paired method--GC differences; DSR-sub reports a single trajectory.

\textbf{Evaluation protocol.}
Table~\ref{tab:downstream-evaluation} lists datasets and metrics. Evaluation occurs every 20 updates, with scores taken from terminal checkpoints. Each problem receives one greedy completion with temperature 0, top-$p$ 1, repetition penalty 1, and a 4096-token limit, stopping at EOS/end-of-turn. Methods share BF16 inference, backend, tokenizer, chat template, and task prompts per model. Validation uses the same decoding and task-specific scorers on the splits in Table~\ref{tab:data-splits}.

\begin{table*}[t]
\centering
\caption{Test sets and evaluation metrics.}
\label{tab:downstream-evaluation}
\renewcommand{\arraystretch}{0.9}
\footnotesize
\begin{tabularx}{\linewidth}{@{}lXrX@{}}
\toprule
Benchmark & Split/version & Problems & Metric \\
\midrule
MATH500~\citep{lightman2023lets} & Full test set & 500 & Final-answer accuracy \\
GSM8K~\citep{cobbe2021training} & \texttt{main} test & 1,319 & Numerical-answer accuracy \\
LiveCodeBench~\citep{jain2024livecodebench} & \texttt{release\_v6}; code generation; full tests & 1,055 & Greedy pass@1 \\
MBPP+~\citep{liu2023evalplus} & v0.2.0 & 378 & Greedy pass@1 \\
K\&K~\citep{xie2025logicrl} & Three- to seven-person test; 100 problems per size & 500 & Macro-averaged complete-assignment accuracy \\
K\&K transfer & Eight-person test & 100 & Complete-assignment accuracy \\
ZebraLogic~\citep{lin2025zebralogic} & Full benchmark & 1,000 & Puzzle-level accuracy \\
\bottomrule
\end{tabularx}
\vspace{-0.2in}
\end{table*}

\textbf{Prompts and scoring.}
Prompts are zero-shot except ZebraLogic's standard one-shot reasoning/JSON template. Math requests reasoning and a boxed answer; code provides the required interface and public examples; K\&K requests all identities. Qwen2.5-Math routines extract, normalize, and equivalence-check answers. Code uses the LiveCodeBench checker (\texttt{--not\_fast}, \texttt{--timeout=6}) and EvalPlus v0.3.1 defaults in matched isolated environments; pass@1 requires all tests, including MBPP+ additions, to pass. K\&K requires exact identities, averaging equally across five 100-problem three-to-seven-person partitions per seed; eight-person accuracy is separate. ZebraLogic requires every grid entry correct. Missing/unparseable answers and errors/timeouts score zero; truncated outputs are scored as returned without changing the denominator.

\subsection{Additional Baselines}
\label{app:full-baseline-comparison}
Table~\ref{tab:full-baseline-comparison} adds activation-compression and selective-recomputation baselines under Table~\ref{tab:main-results}'s protocol. Recovery selection improves in-budget throughput. Under identical memory ceilings, \method gains 10.16--13.47\% over GC across three models, exceeding all added baselines with mean gradient errors of 0.0091--0.0104 (Table~\ref{tab:main-results}). This extends its throughput advantage to activation-compression and selective-recomputation baselines under shared memory and fidelity constraints.

\begin{table*}[t]
\centering
\setlength{\aboverulesep}{1.5pt}
\setlength{\belowrulesep}{1.5pt}
\setlength{\fboxsep}{0.5pt}
\caption{Additional baselines on DeepMath10K with 2K responses.}
\label{tab:full-baseline-comparison}
\scriptsize
\setlength{\tabcolsep}{2.5pt}
\renewcommand{\arraystretch}{0.9}
\begin{tabular}{@{}c|c|ccc|c|cc@{}}
\toprule
\multirow{2}{*}{\textbf{Model}} & \multirow{2}{*}{\textbf{Method}} & \multicolumn{3}{c|}{\textbf{Efficiency}} & \multicolumn{1}{c|}{\textbf{Gradient Fidelity}} & \multicolumn{2}{c@{}}{\textbf{Training Quality}} \\
\cmidrule(lr){3-5}\cmidrule(lr){6-6}\cmidrule(l){7-8}
& & Peak MiB $\downarrow$ & Tok./s $\uparrow$ & Gain (\%) $\uparrow$ & Grad. err. $\downarrow$ & MATH500 $\uparrow$ & GSM8K $\uparrow$ \\
\midrule
\multirow{3}{*}{\textbf{Qwen2.5-3B}}
& GACT\venuetag{ICML'22} & 28,329 & 2,397.52 & $-2.37_{\pm0.39}$ & 0.0146 & $59.27_{\pm0.70}$ & $76.67_{\pm0.31}$ \\
& SAC\venuetag{PyTorch'25} & 30,680 & 2,336.37 & $-4.86_{\pm0.82}$ & 0.0029 & $60.20_{\pm0.60}$ & $77.38_{\pm0.19}$ \\
& CompAct\venuetag{NAACL'25} & 29,058 & 2,346.19 & $-4.46_{\pm0.47}$ & 0.0124 & $59.73_{\pm0.50}$ & $76.90_{\pm0.31}$ \\
\midrule
\multirow{3}{*}{\textbf{Phi-3.5-mini}}
& GACT\venuetag{ICML'22} & 17,338 & 2,695.95 & $+3.28_{\pm0.29}$ & 0.0142 & $43.33_{\pm0.70}$ & $85.27_{\pm0.27}$ \\
& SAC\venuetag{PyTorch'25} & 18,720 & 2,533.33 & $-2.95_{\pm0.72}$ & 0.0029 & $44.07_{\pm0.50}$ & $85.95_{\pm0.19}$ \\
& CompAct\venuetag{NAACL'25} & 17,794 & 2,624.69 & $+0.55_{\pm0.56}$ & 0.0118 & $43.73_{\pm0.70}$ & $85.49_{\pm0.27}$ \\
\midrule
\multirow{3}{*}{\textbf{Llama-3.1-8B}}
& GACT\venuetag{ICML'22} & 38,772 & 1,518.59 & $+1.63_{\pm0.91}$ & 0.0143 & $54.40_{\pm0.60}$ & $82.23_{\pm0.27}$ \\
& SAC\venuetag{PyTorch'25} & 41,520 & 1,410.70 & $-5.59_{\pm0.64}$ & 0.0030 & $55.00_{\pm0.60}$ & $82.87_{\pm0.23}$ \\
& CompAct\venuetag{NAACL'25} & 40,018 & 1,536.67 & $+2.84_{\pm0.18}$ & 0.0122 & $54.80_{\pm0.60}$ & $82.54_{\pm0.27}$ \\
\bottomrule
\end{tabular}
\vspace{-0.2in}
\end{table*}

\subsection{Policy-Update Exposure Details}
\label{app:exposure-details}

\textbf{Post-update diagnostic protocol.} For a reference training update $k$, let $\theta_k^{-}$ denote the policy that generates rollout batch $\mathcal B_k$, and let $\theta_k^{+}$ denote the parameters after the reference learner completes its single optimizer step. We replay $\mathcal B_k$ at $\theta_k^{+}$ while retaining its original advantages, update masks, and old-policy log-probabilities. Diagnostic likelihood ratios are
\begin{equation}
r_{t,k}^{\mathrm{diag}}=\pi_{\theta_k^{+}}(o_t\mid x_t)/\pi_{\theta_k^{-}}(o_t\mid x_t).
\label{eq:diagnostic-ratio}
\end{equation}
where $x_t$ contains prompt and response prefix, and both likelihoods use training probability convention. Each replay evaluates clipped objective at this diagnostic state. Recovery comparisons share identical parameters, rollout inputs, microbatch assignments, and model RNG states. Single-unit L interventions are compared with exact recovery at the same state to obtain $y_{u,b}$. Diagnostic gradients are discarded without an optimizer step, and replay leaves state unchanged.

\textbf{Prediction metrics.} For each held-out microbatch $b$, predictors are evaluated across the same L-eligible units using matched single-unit gradient-error targets $y_{u,b}$. Predictor parameters use only the designated calibration split. Spearman correlation, Top-10\% recall with $k_b=\max(1,\lceil0.1|\mathcal U_{L,b}|\rceil)$, and AUROC for $\mathbf 1[y_{u,b}>\epsilon_{\mathrm{unit}}]$ are computed within each microbatch and summarized within workload, followed by macro-averaging. Spearman uses average ranks for ties; recall uses a fixed target-independent unit order. Constant-score predictors have undefined Spearman correlation and AUROC $0.5$ when both target classes are present. AUROC excludes single-class microbatches.

\textbf{Structural sparsity in diagnostic replays.} Early and late groups aggregate post-update replays associated with reference training updates 9--10 and 11--12, respectively. Sparsity statistics use the response slots in the padded diagnostic records, including trailing padding and excluding prompt slots. Zero-policy-gradient causes are counted exclusively as padding slots, valid response positions with zero advantage, and valid nonzero-advantage positions on a flat clipped policy-loss branch. All three fractions use the padded response-slot count as their denominator. Top-decile concentration ranks the same slots by $|\omega_t|$, with zero coefficients assigned to padding. Exposure scores use the non-padding support $\mathcal T_b$. A nonzero-advantage position is clipped when
\begin{equation}
(\widehat A_t>0\land r_{t,k}^{\mathrm{diag}}>1+\epsilon)\lor(\widehat A_t<0\land r_{t,k}^{\mathrm{diag}}<1-\epsilon).
\label{eq:diagnostic-clipping}
\end{equation}
The zero-policy-gradient fraction rises from 54.8\% to 73.6\%, while the Top-10\% share of complete $|\omega_t|$ rises from 52.4\% to 66.8\%. Complete $\omega_t$ includes the KL contribution at unmasked positions, including those with zero policy-gradient contribution.

\textbf{Prediction signals.} All exposure predictors use the same activation weights from Eq.~\ref{eq:activation-token-weights}. Advantage-only exposure is $\widehat{\mathcal E}^{A}_{u,b}=\sum_{t\in\mathcal T_b}|\widehat A_t|\alpha_{t,u,b}$, complete exposure is $\widehat{\mathcal E}_{u,b}=\sum_{t\in\mathcal T_b}|\omega_t|\alpha_{t,u,b}$, and $\psi_{u,b}=\kappa_{c(u)}\widehat{\mathcal E}_{u,b}\bar e_u^L$ adds calibrated susceptibility and reconstruction distortion. Activation weights are computed from the exact forward pass shared by the matched interventions. Held-out intervention errors are used only as evaluation targets.

\begin{table*}[t]
\centering
\caption{Risk prediction underlying Fig.~\ref{fig:exposure-validation}(b), averaged over held-out replays.}
\label{tab:exposure-per-workload}
\scriptsize
\renewcommand{\arraystretch}{0.9}
\setlength{\tabcolsep}{3.2pt}
\begin{tabular}{lrrrrrrrrr}
\toprule
& \multicolumn{3}{c}{DeepMath10K} & \multicolumn{3}{c}{TACO-Verified} & \multicolumn{3}{c}{Logic-RL} \\
\cmidrule(lr){2-4}\cmidrule(lr){5-7}\cmidrule(lr){8-10}
Signal & $\rho$ & Recall & AUROC & $\rho$ & Recall & AUROC & $\rho$ & Recall & AUROC \\
\midrule
Reconstruction error & 0.28 & 0.37 & 0.63 & 0.24 & 0.35 & 0.61 & 0.26 & 0.36 & 0.62 \\
Static type/layer & 0.32 & 0.42 & 0.66 & 0.29 & 0.39 & 0.64 & 0.30 & 0.40 & 0.65 \\
GACT-style sensitivity & 0.41 & 0.54 & 0.72 & 0.39 & 0.52 & 0.70 & 0.40 & 0.53 & 0.71 \\
ALAM-style sensitivity & 0.44 & 0.56 & 0.73 & 0.40 & 0.54 & 0.72 & 0.42 & 0.55 & 0.72 \\
Advantage only & 0.47 & 0.60 & 0.76 & 0.45 & 0.58 & 0.74 & 0.46 & 0.59 & 0.75 \\
$\widehat{\mathcal E}_u$ & 0.69 & 0.75 & 0.84 & 0.66 & 0.72 & 0.82 & 0.68 & 0.73 & 0.83 \\
$\psi_u$ & \textbf{0.78} & \textbf{0.81} & \textbf{0.89} & \textbf{0.75} & \textbf{0.78} & \textbf{0.86} & \textbf{0.77} & \textbf{0.80} & \textbf{0.88} \\
\bottomrule
\end{tabular}
\vspace{-0.2in}
\end{table*}

\textbf{Exposure and calibrated risk provide successive predictive gains.}
Table~\ref{tab:exposure-per-workload} shows complete exposure raises Spearman by 0.21--0.22 over advantage-only weighting across workloads; susceptibility and reconstruction distortion add 0.09, yielding $\rho=0.75$--$0.78$. Gains support current-update weighting and calibrated sensitivity for identifying high-distortion units.

\begin{wraptable}{r}{0.45\textwidth}
\centering
\vspace{-0.3in}
\caption{Exposure-score ablations.}
\label{tab:exposure-ablation}
\scriptsize
\setlength{\tabcolsep}{1pt}
\begin{tabular}{@{}lrrr@{}}
\toprule
Variant & $\rho$ $\uparrow$ & Top-10\% recall $\uparrow$ & AUROC $\uparrow$ \\
\midrule
Full $\psi_u$ & \textbf{0.78} & \textbf{0.81} & \textbf{0.89} \\
\quad w/o ratio/clipping & 0.65 & 0.68 & 0.81 \\
\quad w/o update mask & 0.78 & 0.81 & 0.89 \\
\quad w/o KL contribution & 0.77 & 0.79 & 0.88 \\
\quad w/o activation weighting & 0.67 & 0.70 & 0.82 \\
\quad w/o susceptibility & 0.74 & 0.76 & 0.86 \\
\bottomrule
\end{tabular}
\vspace{-0.1in}
\end{wraptable}

\textbf{Component ablations.}
Table~\ref{tab:exposure-ablation} removes one score component at a time while fixing diagnostic states, intervention units, target gradients, and the full predictor's calibrated quantities. Removing ratio/clipping replaces $\chi_t r_t$ by $1$ in the policy contribution to the score. The update-mask control sets $m_t=1$ on $\mathcal T_b$ in both policy and KL score contributions while preserving the original normalization. Response-validity masking already assigns this value throughout $\mathcal T_b$, so the control produces identical scores and prediction metrics. Removing KL deletes its contribution to $\omega_t$. Removing activation weighting sets $\alpha_{t,u,b}=1$. Removing susceptibility sets $\kappa_{c(u)}=1$ while retaining $\bar e_u^L$. Exact and perturbed target gradients always use the original objective. These comparisons measure the effect of each component in the calibrated predictor.

\subsection{Canonical Memory--Throughput Frontier}
\label{app:frontier-details}
Table~\ref{tab:frontier} reports Qwen2.5-3B / DeepMath10K with 2K responses, 4096-token packing, and $B_{\mathrm{GC}}=28{,}874$ MiB. Budgets cap measured peak memory; unchanged configurations reuse measurements. \method achieves the highest mean throughput at all tested budgets. At $B=B_{\mathrm{GC}}$, it gains 3.47\% with 8.3\% L and no H. At $1.20B_{\mathrm{GC}}$, the gain reaches 13.97\% as R falls from 91.7\% to 61.1\%, with maximum gradient error 0.0141. Thus, recovery allocation converts memory headroom into less recomputation and higher throughput at fixed packing. Table~\ref{tab:headroom-control} compares this with increased packing.

\begin{table*}[t]
\centering
\caption{Recovery configurations under memory ceilings.}
\label{tab:frontier}
\scriptsize
\renewcommand{\arraystretch}{0.9}
\setlength{\tabcolsep}{2.8pt}
\begin{tabular}{lllrrrrrr}
\toprule
$B/B_{\mathrm{GC}}$ & Budget (MiB) & Method & Peak (MiB) & Tok./s $\uparrow$ & Gain $\uparrow$ & Mean err. $\downarrow$ & Max err. $\downarrow$ & H/L/R (\%) \\
\midrule
1.00 & 28,874 & GC / All-R & 28,874 & 2,455.72 & -- & 0.0000 & 0.0000 & $0/0/100$ \\
& & AGoQ & 28,776 & 2,533.81 & $+3.18\%$ & 0.0110 & -- & -- \\
& & Adacc & 28,764 & 2,476.10 & $+0.83\%$ & 0.0135 & -- & -- \\
& & \method & 28,843 & \textbf{2,540.93} & $\mathbf{+3.47\%}$ & 0.0068 & 0.0074 & $0/8.3/91.7$ \\
\midrule
1.05 & 30,318 & Rockmate & 29,992 & 2,066.73 & $-15.84\%$ & 0.0031 & -- & -- \\
& & AGoQ & 28,776 & 2,533.81 & $+3.18\%$ & 0.0110 & -- & -- \\
& & Adacc & 30,093 & 2,546.83 & $+3.71\%$ & 0.0119 & -- & -- \\
& & \method & 30,269 & \textbf{2,608.96} & $\mathbf{+6.24\%}$ & 0.0087 & 0.0095 & $2.8/11.1/86.1$ \\
\midrule
1.10 & 31,761 & SAC & 30,680 & 2,336.37 & $-4.86\%$ & 0.0029 & -- & -- \\
& & AGoQ & 28,776 & 2,533.81 & $+3.18\%$ & 0.0110 & -- & -- \\
& & Adacc & 31,400 & 2,579.00 & $+5.02\%$ & 0.0107 & -- & -- \\
& & \method & 31,649 & \textbf{2,705.22} & $\mathbf{+10.16\%}$ & 0.0104 & 0.0114 & $8.3/16.7/75.0$ \\
\midrule
1.15 & 33,205 & SAC & 30,680 & 2,336.37 & $-4.86\%$ & 0.0029 & -- & -- \\
& & AGoQ & 28,776 & 2,533.81 & $+3.18\%$ & 0.0110 & -- & -- \\
& & Adacc & 32,811 & 2,616.57 & $+6.55\%$ & 0.0104 & -- & -- \\
& & \method & 33,106 & \textbf{2,765.88} & $\mathbf{+12.63\%}$ & 0.0122 & 0.0133 & $16.7/16.7/66.6$ \\
\midrule
1.20 & 34,649 & SAC & 30,680 & 2,336.37 & $-4.86\%$ & 0.0029 & -- & -- \\
& & AGoQ & 28,776 & 2,533.81 & $+3.18\%$ & 0.0110 & -- & -- \\
& & Adacc & 34,193 & 2,637.69 & $+7.41\%$ & 0.0101 & -- & -- \\
& & \method & 34,552 & \textbf{2,798.78} & $\mathbf{+13.97\%}$ & 0.0130 & 0.0141 & $25.0/13.9/61.1$ \\
\bottomrule
\end{tabular}
\vspace{-0.2in}
\end{table*}

\subsection{Generalization Across Tasks, Lengths, and Models}
\label{app:generalization}

\subsubsection{Cross-Task Transfer}
\label{app:system-robustness}
At $B=1.10B_{\mathrm{GC}}$, \method gains 9.00\% on TACO-Verified and 8.64\% on Logic-RL K\&K with Qwen2.5-3B and 2K responses (Table~\ref{tab:cross-task-system}). It leads the evaluated configurations on both tasks, with paired mean downstream-score differences from GC within 0.10 percentage points across four benchmarks, extending the throughput advantage to code and logical reasoning.

\subsubsection{Response-Length and Model Transfer}
\label{app:scaling-details}

Fig.~\ref{fig:length-model-transfer} evaluates response-length and model transfer using model-specific GC references. Panels~(a)--(c) compare 1K, 2K, and 4K responses on Qwen2.5-3B and Phi-3.5-mini under $B=1.10B_{\mathrm{GC}}$. Panels~(d)--(f) compare three models with 4K responses under the tighter $B=1.05B_{\mathrm{GC}}$.

\begin{figure}[t]
    \centering
    \includegraphics[width=0.95\linewidth]{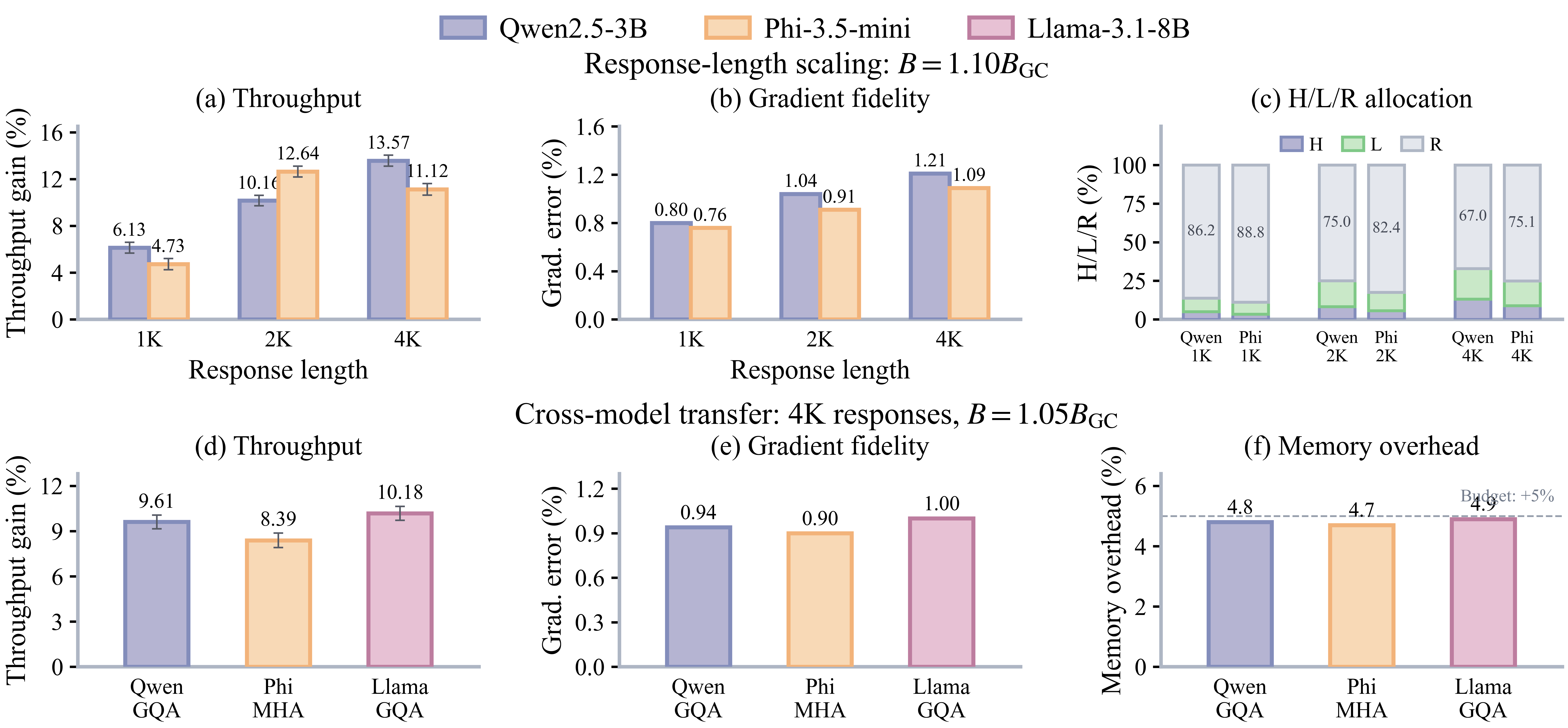}
    \vspace{-0.2in}
    \caption{Response-length and model transfer. Throughput gains are relative to paired GC, with error bars showing sample SDs. Gradient errors and additional peak memory are expressed as percentages.}
    \label{fig:length-model-transfer}
    \vspace{-0.1in}
\end{figure}

\textbf{Recovery gains persist across response lengths.}
All six settings improve throughput over paired GC, with gains of 4.73--13.57\% across 1K--4K responses and mean relative gradient errors at or below 1.21\% (Fig.~\ref{fig:length-model-transfer}a,b). Qwen's gain increases from 6.13\% at 1K to 13.57\% at 4K, while Phi achieves its largest gain at 2K. Thus, the benefit persists across the tested lengths without requiring throughput gains to increase monotonically with response length.

\textbf{Recovery allocation shifts toward retention at longer responses.}
Fig.~\ref{fig:length-model-transfer}c shows increasing H and L fractions as response length grows. From 1K to 4K, the R fraction decreases from 86.2\% to 67.0\% on Qwen and from 88.8\% to 75.1\% on Phi. These allocations show how the selected recovery policies vary with response length while maintaining the reported gradient fidelity.

\textbf{Gains extend across attention architectures under tighter budgets.}
Across the three GQA/MHA models, \method improves throughput by 8.39--10.18\%, with mean relative gradient errors at or below 1.00\% (Fig.~\ref{fig:length-model-transfer}d,e). Additional peak memory is 4.8\%, 4.7\%, and 4.9\% for Qwen, Phi, and Llama, respectively, remaining below the shared 5\% allowance (Fig.~\ref{fig:length-model-transfer}f). These results support an efficiency benefit across the tested architectures with modest memory headroom.

\subsubsection{Training Quality and Full-Parameter Transfer}
\label{app:quality-full-parameter-transfer}

\begin{wraptable}{r}{0.48\textwidth}
\centering
\vspace{-0.32in}
\caption{Paired downstream quality changes relative to GC, in percentage points.}
\label{tab:quality-generalization-paired}
\footnotesize
\renewcommand{\arraystretch}{0.9}
\setlength{\tabcolsep}{4pt}
\resizebox{\linewidth}{!}{%
\begin{tabular}{@{}lcc@{}}
\toprule
Setting & $\Delta$ Metric 1 & $\Delta$ Metric 2 \\
\midrule
TACO-Verified / Qwen2.5-3B & $-0.06_{\pm0.36}$ & $+0.09_{\pm0.40}$ \\
Logic-RL K\&K / Qwen2.5-3B & $+0.07_{\pm0.31}$ & $+0.10_{\pm0.17}$ \\
Phi-3.5-mini / DeepMath10K & $+0.53_{\pm0.50}$ & $+0.13_{\pm0.19}$ \\
Llama-3.1-8B / DeepMath10K & $+0.53_{\pm0.46}$ & $+0.23_{\pm0.23}$ \\
\bottomrule
\end{tabular}}
\vspace{-0.2in}
\end{wraptable}

Table~\ref{tab:quality-generalization-paired} reports mean $\pm$ sample SD of within-seed method--GC score differences in percentage points, computed before rounding. Metric~1/2 denotes LiveCodeBench/MBPP+ for TACO-Verified, three- to seven-person K\&K/ZebraLogic for Logic-RL, and MATH500/GSM8K for DeepMath10K. Paired differences remain within 0.53 percentage points across settings, supporting downstream quality close to GC at operating points used for throughput comparisons. The single-seed full-parameter experiment in Fig.~\ref{fig:dsr-learning} uses Llama-3.2-3B-Instruct on DSR-sub with 3K responses. GC improves from 9.60\% to 47.80\% MATH500 accuracy, while \method reaches 48.40\%, extending the learning-quality evidence beyond LoRA.

\subsubsection{Full-Parameter Memory--Throughput Trade-offs}
\label{app:full-parameter-budget-planning}

\textbf{Setup.}
Fig.~\ref{fig:full-parameter-runtime} evaluates Llama-3.2-3B-Instruct / DSR-sub full-parameter GRPO on one A100 80GB, with GPU-resident FP32 parameters, gradients, and AdamW moments; BF16 computation; FlashAttention-backed SDPA; padding removal; and 1024/3072-token prompt/response limits. Fixed GC uses 8192 valid tokens per microbatch, yielding $B_{\mathrm{GC}}=59.37$ GiB and $v_{\mathrm{GC,fixed}}=5{,}452$ tokens/s. For $B/B_{\mathrm{GC}}\in\{1.00,1.05,\ldots,1.20\}$, tuned GC adjusts microbatches with decoder-layer checkpointing; \method selects recovery and targets from $\{8192,10240,12288\}$. Inputs and loss normalization remain fixed. Actor-update timing excludes rollouts and old-policy/reference log-probability computation. Valid prompt/response tokens are counted once, with $t_{10^6,a}=10^6/v_a$ and $\mathrm{Gain}_a=100(v_a/v_{\mathrm{GC,fixed}}-1)$. Fig.~\ref{fig:dsr-learning} learning curves use the four-rank configuration in Table~\ref{tab:execution-configurations}.

\begin{figure}[t]
\centering
\includegraphics[width=\linewidth]{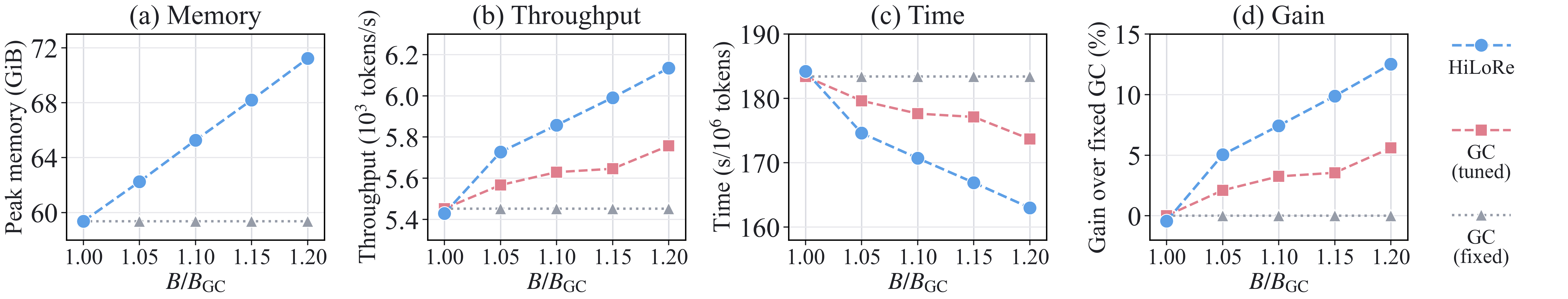}
\vspace{-0.3in}
\caption{Full-parameter memory--throughput trade-offs. GC (tuned) adjusts microbatch size per budget with decoder-layer checkpointing.}
\label{fig:full-parameter-runtime}
\vspace{-0.2in}
\end{figure}

\textbf{Results.}
At $1.05$--$1.20B_{\mathrm{GC}}$, \method improves throughput by 5.04--12.52\% over fixed GC and 2.88--6.56\% over tuned GC. Joint recovery and microbatch allocation thus converts memory headroom into higher throughput than microbatch tuning alone in this full-parameter setting.

\subsection{Sensitivity to the Shared Gradient-Error Tolerance}
\label{app:shared-tolerance}

We compare \method, Adacc, AGoQ, PRAC, and INSTANT on Qwen2.5-3B / DeepMath10K with rank-16 LoRA, 2K responses, and $\epsilon_g\in\{0.0100,0.0125,0.0150,0.0175,0.0200\}$, fixing $B=1.10B_{\mathrm{GC}}$ and the terminal validation-score requirement. Each method uses a fixed search space covering all tolerances, capped at 16 distinct configurations. Applying the pilot finalist rule at each $\epsilon_g$ yields a union for complete development evaluation. Each tolerance selects from this shared pool using Eq.~\ref{eq:configuration-selection} and the throughput tie rule. Independent fixed-workload evaluation follows Appendix~\ref{app:experimental-setup}, reusing measurements for repeated selections.

\begin{wraptable}{r}{0.5\textwidth}
\centering
\vspace{-0.3in}
\caption{Shared gradient-error tolerance sensitivity. Each cell reports throughput gain over GC (\%) / mean gradient error (\%).}
\label{tab:shared-tolerance-sensitivity}
\scriptsize
\setlength{\tabcolsep}{3pt}
\resizebox{\linewidth}{!}{%
\begin{tabular}{@{}lccccc@{}}
\toprule
$\epsilon_g$ & \method & Adacc & AGoQ & PRAC & INSTANT \\
\midrule
1.00\% & 8.76 / 0.78 & 4.30 / 0.95 & 2.55 / 0.93 & 1.85 / 0.94 & 1.70 / 0.96 \\
1.25\% & 10.16 / 1.04 & 5.02 / 1.07 & 3.18 / 1.10 & 2.76 / 1.12 & 2.43 / 1.05 \\
1.50\% & 10.16 / 1.04 & 5.02 / 1.07 & 3.18 / 1.10 & 2.76 / 1.12 & 2.43 / 1.05 \\
1.75\% & 10.94 / 1.58 & 7.10 / 1.68 & 6.20 / 1.73 & 6.70 / 1.65 & 7.40 / 1.67 \\
2.00\% & 11.28 / 1.81 & 8.10 / 1.93 & 7.20 / 1.96 & 8.70 / 1.94 & 9.70 / 1.91 \\
\bottomrule
\end{tabular}}
\vspace{-0.12in}
\end{wraptable}

\textbf{Throughput advantage across nearby tolerances.}
Table~\ref{tab:shared-tolerance-sensitivity} shows \method gains 8.76--11.28\% over GC across five tolerances, exceeding the strongest baseline by 1.58--5.14 percentage points. The 1.25\% and 1.50\% settings select identical configurations, retaining a 10.16\% gain with 1.04\% mean gradient error. The throughput advantage remains stable around the default tolerance.

\subsection{Ablation, Sensitivity, and Runtime Details}
\label{app:ablation-robustness}

\begin{wraptable}{r}{0.49\textwidth}
\centering
\vspace{-0.3in}
\caption{Update-regime statistics averaged over held-out post-update diagnostic replays.}
\label{tab:update-regime-analysis}
\setlength{\tabcolsep}{2pt}
\resizebox{\linewidth}{!}{%
\begin{tabular}{@{}lrrrrr@{}}
\toprule
Regime & Norm. $|\omega_t|$ & Norm. exposure & H (\%) & L (\%) & L drift \\
\midrule
Low $|\widehat A_t|$ & 0.41 & 0.48 & 4.8 & 18.5 & 0.0062 \\
High $|\widehat A_t|$ & 1.84 & 1.71 & 14.7 & 9.4 & 0.0127 \\
Clipped & 0.19 & 0.27 & 2.9 & 21.3 & 0.0051 \\
Active unclipped & 1.37 & 1.43 & 12.6 & 11.2 & 0.0115 \\
\bottomrule
\end{tabular}}
\vspace{-0.2in}
\end{wraptable}

\textbf{System-ablation protocol.}
Qwen2.5-3B / DeepMath10K ablations use pre-update replay protocol with $B_{\mathrm{GC}}=28{,}874$ MiB and $B=31{,}761$ MiB, $\delta^\star$, 16 calibration microbatches, and per-microbatch exposure/risk refresh. Block recovery uses Transformer blocks; size-based/uniform H replace utility; \method-HR uses H/R. Budget sweeps retain capture costs when $\delta=0$ disables L. Errors reference held-out All-R.

\textbf{Diagnostic update regimes.}
Table~\ref{tab:update-regime-analysis} uses the post-update protocol in Appendix~\ref{app:exposure-details} with same memory budget and allocation settings. Candidate plans, units, profiled utilities, and calibrated quantities are restored for comparisons; current coefficients, exposure scores, and recovery actions are recomputed at diagnostic state. Low/high advantage denotes the bottom/top $|\widehat A_t|$ quartiles. Clipped positions satisfy $m_t=1$, $\widehat A_t\ne0$, and Eq.~\ref{eq:diagnostic-clipping}; active-unclipped positions satisfy $m_t=1$, $\widehat A_t\ne0$, and $\chi_t=1$. Complete coefficients retain KL contribution. Coefficients use rollout-batch normalization. For regime $G$, define $w_{u,b,G}=\sum_{t\in G_b}\alpha_{t,u,b}$. H/L fractions are weighted means of the corresponding unit-action indicators with weights $w_{u,b,G}$ across diagnostic records. Normalized exposure is the activation-weighted mean of $|\omega_t|$ over token--unit pairs in $G$, divided by the corresponding mean over all response token--unit pairs. L drift is the $w_{u,b,G}$-weighted mean of $y_{u,b}$ over L-selected units. Each $y_{u,b}$ measures the complete gradient deviation from a whole-unit intervention.

\textbf{Sensitivity.}
Sweeps change one factor while fixing replay inputs, candidate plans, units, utilities, memory budgets, and held-out updates. Calibration uses nested 32-microbatch subsets. With 16 calibration microbatches, dynamic scores refresh every $r\in\{1,4,8\}$ microbatches within each unit/shape class and on cache misses. A refresh collects activation weights and recomputes exposure from current token coefficients; intermediate microbatches reuse the cached per-unit risk scores. Static risk uses each unit's mean calibrated risk over the initialization microbatches. Susceptibility and reconstruction statistics remain fixed within each run. Allocation runs every microbatch in all configurations, and timing includes the collection work performed by each refresh schedule.

\begin{table*}[t]
\centering
\caption{Sensitivity to risk budget, calibration size, and refresh interval. Scheduling overhead includes exposure/risk refresh and allocation.}
\label{tab:delta-sensitivity}
\label{tab:calibration-size}
\label{tab:refresh-sensitivity}
\scriptsize
\setlength{\tabcolsep}{2pt}
\renewcommand{\arraystretch}{0.9}
\begin{tabular}{llrrrrr}
\toprule
Sweep & Setting & Gain (\%) $\uparrow$ & Mean err. $\downarrow$ & Max err. $\downarrow$ & Scheduling overhead (\%) & H/L/R (\%) \\
\midrule
\multirow{6}{*}{$\delta/\delta^\star$}
& 0 & $+6.14\pm0.48$ & 0.0030 & 0.0034 & 0.19 & $13.9/0/86.1$ \\
& 0.25 & $+7.35\pm0.47$ & 0.0057 & 0.0064 & 0.21 & $11.1/5.6/83.3$ \\
& 0.5 & $+8.76\pm0.46$ & 0.0078 & 0.0088 & 0.21 & $8.3/11.1/80.6$ \\
& 1 & $+10.16\pm0.44$ & 0.0104 & 0.0114 & 0.21 & $8.3/16.7/75.0$ \\
& 2 & $+10.94\pm0.46$ & 0.0158 & 0.0176 & 0.21 & $5.6/22.2/72.2$ \\
& 4 & $+11.28\pm0.49$ & 0.0181 & 0.0206 & 0.22 & $2.8/30.6/66.6$ \\
\midrule
\multirow{4}{*}{Calibration batches}
& 4 & $+8.92\pm0.52$ & 0.0119 & 0.0157 & 0.21 & $11.1/11.1/77.8$ \\
& 8 & $+9.67\pm0.48$ & 0.0110 & 0.0128 & 0.21 & $8.3/13.9/77.8$ \\
& 16 & $+10.16\pm0.44$ & 0.0104 & 0.0114 & 0.21 & $8.3/16.7/75.0$ \\
& 32 & $+10.18\pm0.45$ & 0.0103 & 0.0115 & 0.21 & $8.3/16.7/75.0$ \\
\midrule
\multirow{4}{*}{Refresh interval}
& Every microbatch & $+10.16\pm0.44$ & 0.0104 & 0.0114 & 0.21 & $8.3/16.7/75.0$ \\
& Every 4 microbatches & $+10.04\pm0.45$ & 0.0110 & 0.0132 & 0.12 & $8.3/13.9/77.8$ \\
& Every 8 microbatches & $+9.72\pm0.48$ & 0.0118 & 0.0147 & 0.11 & $11.1/11.1/77.8$ \\
& Static & $+8.31\pm0.51$ & 0.0128 & 0.0161 & 0.09 & $11.1/8.3/80.6$ \\
\bottomrule
\end{tabular}
\vspace{-0.2in}
\end{table*}

Raising $\delta$ from zero to $\delta^\star$ increases gain from 6.14\% to 10.16\% at mean error 0.0104; both larger tested budgets exceed $\epsilon_g=0.015$ (Table~\ref{tab:delta-sensitivity}). Doubling calibration from 16 to 32 microbatches adds 0.02 percentage points. Per-microbatch refresh improves gain over static risk from 8.31\% to 10.16\% and lowers maximum error from 0.0161 to 0.0114 at 0.21\% overhead, supporting current-update adaptation. Refreshing every four microbatches retains 10.04\% gain at 0.12\% overhead.

\begin{wraptable}{r}{0.45\textwidth}
\centering
\vspace{-0.3in}
\caption{Profiled actor-update components in seconds; totals sum the listed components.}
\label{tab:overhead}
\setlength{\tabcolsep}{3pt}
\resizebox{\linewidth}{!}{%
\begin{tabular}{@{}lrrr@{}}
\toprule
Component & GC & \method-HR & \method \\
\midrule
Forward & 146.72 & 146.85 & 146.91 \\
Recomputation & 207.38 & 164.10 & 118.64 \\
Candidate encoding & -- & -- & 6.37 \\
L reconstruction & -- & -- & 7.88 \\
Exposure refresh & -- & 0.72 & 0.74 \\
Allocation & -- & 0.51 & 0.58 \\
Backward & 328.79 & 331.21 & 338.79 \\
\midrule
Profiled total & $682.89_{\pm4.22}$ & $643.39_{\pm3.88}$ & $619.91_{\pm3.74}$ \\
\bottomrule
\end{tabular}}
\vspace{-0.15in}
\end{wraptable}

\textbf{Runtime breakdown.}
Five warmed-up repetitions use synchronized, non-overlapping component timing. Backward includes gradient computation/communication, gradient-norm clipping, and optimizer work. Encoding covers all captured candidates, including those discarded; allocation covers plan and recovery selection. Offline profiling, initial calibration, and reference-gradient collection are timed separately. Recomputation savings (88.74 s) exceed encoding, reconstruction, refresh, and allocation costs (15.57 s). Including forward/backward changes, the profiled total falls by 9.22\% relative to GC and saves 23.48 s beyond H/R alone (Table~\ref{tab:overhead}).

\section*{Limitations}
\label{app:statements}
\method estimates gradient distortion through calibrated exposure proxies and additive risk scores motivated by first-order perturbation analysis. Complete-update gradient checks assess the combined effects of calibration error and higher-order interactions. Calibration statistics remain fixed during training, while exposure and risk scores refresh online. Evaluations primarily cover rank-16 LoRA with responses up to 4K on 48-GB GPUs; full-parameter evidence covers Llama-3.2-3B-Instruct / DSR-sub through single-seed learning curves and single-A100 efficiency measurements.

\end{document}